\documentclass[twoside,11pt]{article} 
\usepackage{jmlr2e}
\usepackage{natbib}

\usepackage{amsmath,amssymb} 
\usepackage{url,graphicx,algorithm2e, hyperref, nameref,xkeyval}

\usepackage{xcolor}
\usepackage{mathrsfs,stackengine} 
\usepackage{amsfonts}
\usepackage{mathtools}
\usepackage{epstopdf}
\usepackage{ifthen}
\usepackage{bbm}
\graphicspath{{eps/}}

\newcommand{\inR}{\in \mathbb{R}}

\newcommand{\R}{ \mathbb{R}}

\newcommand{\N}{ \mathbb{N}}

\newcommand{\eqdef}{\stackrel{\vartriangle}{=}}

\newcommand{\Lop}{{\rm L}}
\newcommand{\Dop}{{\rm D}}

\newcommand{\dint}{{\rm d}}

\newcommand{\Fourier}{ \mathcal{F}}

\DeclareMathOperator*{\esssup}{ess\,sup}

\def\V#1{{\boldsymbol{#1}}}         % vectors
\def\Spc#1{{\mathcal{#1}}}  % spaces
\def\M#1{{\bf{#1}}}  % matrices
\def\Op#1{{\mathrm{#1}}}  % operator
\def\ee{\mathrm{e}} 
\def\jj{\mathrm{i}} 
\def\Indic{\mathbbm{1}} 
\def\One{\mathbbm{1}}

\newcommand{\eq}[1]{{\rm(\ref{#1})}}

\newcommand{\embedC}{\xhookrightarrow{}}

\providecommand{\revise}[1]{#1}

\begin{document}

\title{A representer theorem for deep neural networks}

%\author{Michael Unser\thanks{Biomedical Imaging Group, \'Ecole polytechnique f\'ed\'erale de Lausanne (EPFL), CH-1015 Lausanne, Switzerland}}

%\editor{Lorenzo Rosasco}
\editor{}

\author{\name {Michael Unser} \email {michael.unser@epfl.ch}\\
\addr Biomedical Imaging Group,\\ \'Ecole Polytechnique F\'ed\'erale de Lausanne (EPFL),\\ CH-1015 Lausanne, Switzerland
}

%\significancestatement{
%Deep neural networks achieve state-of-the-art performance in many applications of machine learning. 
%Since there is limited theory available to support these techniques, the current research efforts in this fast-growing area are empirical for the most part. The contribution of this paper is theoretical. It proposes a new variational optimization of neural activations that yields the first representer theorem for deep neural networks. The outcome is a deep spline architecture that is compatible with state-of-the-art ReLU and MaxOut networks, but that calls for some adaptations in parameterization and in training procedure.
%}
%\keywords{Deep learning | neural networks | splines | regularization | sparsity | functional optimization}

%\dates{This manuscript was compiled on \today}
%\doi{\url{www.pnas.org/cgi/doi/10.1073/pnas.XXXXXXXXXX}}

%\abbreviations{gTV, generalized total variation}

%\verticaladjustment{-2pt}
\maketitle

\begin{abstract}

We propose to optimize the activation functions of a deep neural network by adding a corresponding functional regularization to the cost function. We justify the use of a second-order total-variation criterion. This allows us to derive a general representer theorem for deep neural networks that makes a direct connection with splines and sparsity. Specifically, we show that the optimal network configuration
can be achieved with activation functions that are nonuniform linear splines with adaptive knots.
The bottom line is that the action of each neuron is encoded by a spline whose parameters (including the number of knots) are optimized during the training procedure. The scheme results in a computational structure that is compatible with existing deep-ReLU, parametric ReLU, APL (adaptive piecewise-linear) and MaxOut architectures.
It also suggests novel optimization challenges, while making the link with $\ell_1$ minimization and sparsity-promoting techniques
explicit.

\end{abstract}


\section{Introduction}

The basic regression problem in machine learning is to find a parametric representation 
of a function $f: \R^N \to \R$ given a set of data points $(\V x_m,y_m)\inR^{N +1}$ such that $f(\V x_m)$ is close to $y_m$
for $m=1,\dots,M$ in an appropriate sense \citep{Bishop2006}. Classically, there are two steps involved. The first is the {\em design}, which can be 
abstracted in the choice of a given parametric class of functions $\V x \mapsto f(\V x|\V \theta)$, where $\V \theta$ encodes the parameters. For instance,
$f(\V x|\V \theta)$ could be a neural network with weights $\V \theta$.
The second is the {\em training}, which basically amounts to an interpolation/approximation problem where the chosen model is fit to the data. In practice, the optimal parameter $\V \theta_0$ is determined via the functional minimization% of a functional of the form
\begin{align}
\label{Eq:training}
\V \theta_0=\arg \min_{\V \theta} \sum_{m=1}^M E\big(y_m,f(\V x_m|\V \theta)\big),
\end{align}
where $E: \R \times \R \to \R^+$ is a convex error function that quantifies the discrepancy of the fit to the data.
% with the property that $E(y,y)=0$ for all $y\inR$.
A classical choice is $E\big(y_m,f(\V x_m|\V \theta)\big)=|y_m-f(\V x_m|\V \theta)|^2$, which yields the least-squares solution.

The most delicate step is the design, because it has to deal with two conflicting requirements.
First is the desire for universality, meaning that
the parametric model $f(\V x|\V \theta)$ %---say, a neural net with adjustable weights $\V \theta$---
should be flexible enough %and sufficiently rich 
to allow for the faithful representation of a large class of functions---ideally, the complete family of continuous functions
$\R^N\to \R$, as the dimensionality of $\V \theta$ goes to infinity. Second is the quest for parsimony, 
meaning that the model should have a small number of parameters, which leads to an increase in robustness and trustworthiness.

This work aims at unifying the design of neural networks based on %rigorous 
variational principles inspired by kernel methods. To set up the stage, we now briefly review %these 
the two relevant approaches to supervised learning. %  \citep{Bishop2006}.
\subsection{Kernel methods}
A kernel estimator is a linear model with adjustable parameters $\V \theta=(a_1,\dots,a_M)\in\R^M$ 
and predefined data centers $\V x_1, \dots,\V x_M \in \R^N$ of the form
\begin{align}
\label{Eq:Kernelf}
f(\V x|\V \theta)=\sum_{m=1}^M a_m h(\V x,\V x_m),
\end{align}
where $\V x \inR^N$ is the input variable of the model and where $h: \R^N \times \R^N \to \R$ is a positive-definite kernel, a preferred choice being the Gaussian kernel $h(\V x,\V y)=\ee^{-\tfrac{1}{2}\|\V x-\V y\|^2/\sigma^2}$ \citep{Hofmann2008, Alvarez2012}.
This expansion is at the heart of the whole class of kernels methods, including radial-basis functions and support-vector machines \citep{Scholkopf1997,Vapnik2013,Scholkopf2002}. %, which used to constitute the state-of-the-art in the field until rather recently (that is, prior to the deep neural network revolution).

The elegance of kernel estimators lies in that they can be justified based on regularization theory \citep{Poggio1990,Evgeniou2000, Poggio2003}. 
The incentive there is to remove some of the arbitrariness of model selection by formulating the learning task as a global minimization problem that takes care of the design and training jointly. %The added difficulty is that this requires some optimization over function spaces. 
The property that makes such an integrated approach feasible is that any 
Hilbert space $\Spc H$ of continuous functions on $\R^N$ has a unique {\em reproducing kernel} $h_\Spc H: \R^N \times \R^N \to \R$ such that
({\em i}) $h_\Spc H(\cdot,\V x_m) \in \Spc H$; and ({\em ii}) $\langle f,h_\Spc H(\cdot,\V x_m)\rangle_{\Spc H}=f(\V x_m)$ 
for any $\V x_m \in \R^N$ and $f \in \Spc H$ \citep{Aronszajn1950}. The idea, then, is
to formulate the
``regularized'' version of Problem \eq{Eq:training} as
\begin{align}
\label{Eq:RKHS}
 f_{\rm RKHS}=\arg \min_{f \in \Spc H} \left(\sum_{m=1}^M E\big(y_m,f(\V x_m)\big) + \lambda \|f\|^2_{\Spc H}\right),
\end{align}
where the second term penalizes solutions with a large $\|\cdot\|_\Spc H$-norm and  $\lambda\inR^+$ is an adjustable tradeoff factor.
Under the assumption that the loss function $E$ is convex, the representer theorem \citep{Kimeldorf1971c,Scholkopf2001,Scholkopf2002} then states that the solution of \eq{Eq:RKHS}
exists, is unique, and such that  $f_{\rm RKHS}\in {\rm  span}\{h_\Spc H(\cdot,\V x_m)\}_{m=1}^M$. This ultimately results in the same linear expansion as \eq{Eq:Kernelf}. The argument also applies the other way round since any positive-definite kernel $h$ 
specifies a unique reproducing-kernel Hilbert space (RKHS) $\Spc H_h$, which then provides the regularization functional $\|f\|_{\Spc H_h}^2$ in \eq{Eq:RKHS} that is matched to the kernel estimator specified by \eq{Eq:Kernelf}. 

The other remarkable feature  of kernel expansions is their {\em universality}, under mild conditions on $h$ \citep{Micchelli2006}.
% (i.e., a stricter version of positive definiteness). 
In other words, one has the guarantee that the generic linear model of \eq{Eq:Kernelf} can reproduce {\em any} continuous function $f :\R^N\to \R$ to a desired degree of accuracy by including sufficiently many centers, with the error vanishing as $M\to \infty$. Moreover, because of the tight connection between kernels, RKHS,  and splines \citep{deBoor1966, Micchelli1986,Wahba1990}, one can invoke standard results in approximation theory to obtain quantitative estimates of the approximation error of smooth functions as a function of $M$ and of the widest gap between data centers \citep{Wendland2005}.
Finally, there is a well-known link between kernel methods derived from regularization theory and neural networks, albeit ``shallow'' ones that involve a single nonlinear layer \citep{Poggio1990}.

\subsection{Deep neural networks}
While kernel methods have been a major (and winning) player in machine learning since the mid '90s, they have been recently outperformed by deep neural networks (DNNs) in many real-world applications such as
image classification \citep{Krizhevsky2012}, speech recognition \citep{Hinton2012}, and image segmentation \citep{Ronneberger2015}. %, to name a few.

The leading idea of deep learning is to build more powerful learning architectures via the stacking/composition of simpler entities (see the review papers by LeCun, Bengio and Hinton \citep{LeCun2015} and Schmidhuber \citep{Schmidhuber2015} and the recent textbook \citep{Goodfellow2016} for more detailed explanations).
In this work, we focus on the popular class of feedforward
networks that involve a layered composition of affine transformations (linear weights) and pointwise nonlinearities.
The deep structure of such a network is specified by its {\em node descriptor} $(N_0,N_1, \dots,N_L)$ where $L$ is the total number of layers (depth of the network) and $N_\ell$ is the number of neurons at the $\ell$th layer. % (see Fig. 1).
The action of a (scalar) neuron (or node) indexed by $(n,\ell)$ is described by the relation
$\sigma(\M w_{n,\ell}^T\V x - b_{n,\ell})$ where $\V x\inR^{N_{\ell-1}}$ denotes the multivariate input of the neuron, % \big(=output of layer $(\ell-1)$\big), 
$\sigma: \R \to \R$ is a predefined activation function (such as a sigmoid or a ReLU=rectified linear unit), $\M w_{n,\ell} \inR^{N_{\ell-1}}$ a set of linear weights, and $b_{n,\ell}\inR$ an additive bias. The outputs of layer $\ell$ are then fed as inputs of layer $(\ell+1)$, and so forth for $\ell=1,\dots,L$.

To obtain a global description, we group the neurons within a given layer $\ell$ 
and specify the two corresponding vector-valued maps:
\begin{enumerate}
\item Linear step $\V f_\ell: \R^{N_{\ell-1}} \to \R^{N_{\ell}}$ (affine transformation)
%Specifically, given a series of integers $N_0=d, N_1,\dots,N_{\ell}$ representing the number of nodes of the network in each layer, we define
%the series of affine transformation $\M \sigma_1: \R^{N_0} \to \R^{N_1}, \dots, \M \sigma_n: \R^{N_{\ell-1}} \to \R^{N_{\ell}},\dots,\M \sigma_N:\R^{N_{\ell-1}} \to \R^N$
\begin{align}
\label{eq:CPWL}
\V f_\ell: \V x\mapsto \V f_\ell(\V x)=\M W_\ell\V x - \M b_{\ell} 
\end{align}
with weighting matrix $\M W_\ell=[\M w_{1,\ell} \cdots \M w_{N_\ell,\ell}]^T\inR^{N_\ell \times N_{\ell-1}}$ and bias vector $\M b_{\ell}=(b_{1,\ell}, \dots,b_{N_\ell,\ell}) \inR^{N_\ell}$.
\item Nonlinear step $\V \sigma_\ell : \R^{N_\ell} \to \R^{N_\ell}$ (activation functions)
\begin{multline}
\V \sigma_\ell :\V x=(x_1,\dots,x_{N_\ell}) %\\ 
\mapsto \V \sigma_\ell(\V x)=\big(\sigma_{1,\ell}(x_1),\dots,\sigma_{N_\ell,\ell}(x_{N_\ell})\big)
\end{multline}
with the possibility of adapting the scalar activation functions $\sigma_{n,\ell}$ on a per-node basis.
\end{enumerate}
This allows us to describe the overall action of the full $L$-layer deep network by %a single equation
\begin{align}
\label{eq:DeepNet}
\M f_{\rm deep}(\V x)=\left(\V \sigma_L \circ \V f_L \circ \V \sigma_{L-1} \circ \cdots \circ \V \sigma_2 \circ \V f_2 \circ \V \sigma_1 \circ \V f_1\right)(\V x),
\end{align}
which makes its compositional structure explicit. The design step therefore consists in fixing the architecture of the deep neural net: One must specify $(N_0,N_1, \dots,N_L)$ together with the activation functions $\V \sigma_{\ell}: \R^{N_\ell} \to \R^{N_\ell}$. The activations are traditionally chosen to be not only the same for all neurons within a layer, but also the same across layers. This results in a computational structure  with adjustable parameters $\V \theta=(\M W_1, \dots, \M W_L, \M b_1,\dots,\M b_L)$ (weights of the linear steps). These are then set during training via the minimization of \eq{Eq:training}, which 
is achieved by stochastic gradient descent with efficient error backpropagation \citep{Rumelhart1986}.

%\subsection{First hints to spline connection}
While researchers have considered a variety of possible activation functions, such as the traditional sigmoid, a preferred choice that has emerged over the years is the rectified linear unit: ${\rm ReLU}(x)\eqdef \max(x,0)$ \citep{Glorot2011}.
The reasons that support this choice are multiple. The initial motivation was to promote sparsity (in the sense of decreasing the number of active units), capitalizing on the property that ReLU acts as a gate and works well in combination with $\ell_1$-regularization \citep{Glorot2011}.
Second is the empirical observation that the training of very deep networks is much faster 
if the hidden layers are composed of ReLU activation functions \citep{LeCun2015}. 
% Third is the tight link with the popular MaxOut and max pooling operations, which can also be implemented using ReLUs \citep{Goodfellow2013}.
Last but not least is the connection between deep ReLU networks and splines---to be further developed in this paper. %---aspects of which have been uncovered by several research groups. 

A key observation %in that respect 
is that a deep ReLU network implements a multivariate input-output relation that is continuous and piecewise-linear (CPWL) \citep{Montufar2014}. 
This remarkable property is due to the ReLU itself being a linear spline, which has prompted Poggio {\em et al.}\ to interpret deep neural networks as hierarchical splines \citep{Poggio2015}. Moreover, it has been shown that any CPWL function admits a deep ReLU implementation
\citep{Wang2005, Arora2016}, which is quite significant since the CPWL family has universal approximation properties. % \citep{Goodfellow2013}.

The ability of splines to effectively represent arbitrary (univariate) functions \citep{deBoor1978,Schumaker1981,Unser1999} has also been exploited at the more local level of a neuron/node in a network. Several authors have proposed to use spline-related parametric models to optimize the shape of neural activation units. Existing designs include B-spline receptive fields \citep{Lane1991}, Catmul-Rom splines \citep{Vecci1998}, cubic spline activations \citep{Guarnieri1999}, adaptive piecewise-linear (APL) units \citep{Agostinelli2015}, and smooth piecewise-polynomial
functions \citep{Hou2017}.

\subsection{Road map}
Our purpose in this paper is to strengthen the connection between splines and multilayer ReLU networks even further.
To that end, 
%we start by reviewing the primary properties of CPWL functions in Section II, while making the point that these are
%the natural multivariate extension of the non-uniform splines of degree 1. This ground material is essential for understanding the functional equivalence between  CPWL functions and deep ReLU networks. In Section III, we 
we formulate the design of a deep neural network globally within the context of regularization theory, in direct analogy with the variational formulation of kernel estimators given by \eq{Eq:RKHS}. 
The critical aspect, of course, is the selection of an appropriate regularization functional which, for reasons that will be exposed next, will take us outside of the traditional realm of RKHS.

Having set the deep architecture of the neural network, we then formulate
the training as a global optimization task whose outcome is a combined set of optimal neuronal activation functions and linear weights.
The foundational role of the representer theorem (Theorem \ref{Theo:CNNRep}) is that it will provide us with the parametric representation
of the optimal activations, which can then be %5the basis for 
leveraged for obtaining a numerical implementation that is compatible with current architectures; in particular, the popular deep RELU networks.

%spline activations.
%By applying a variational criterion (2nd-order TV) to constrain the behavior of the activation functions, we shall then revisit
%We then derive the corresponding representer theorem  that resembles the one known for kernel methods. The main outcome is that the optimal activation functions are linear splines with {\em adaptive} knots. 
%This yields a computational structure that is compatible with standard ReLU and/or MaxOut units together with a optimization strategy that favors ``sparse'' solutions; that is, spline activations with few knots.

\section{From deep neural networks to deep splines}

Given the generic structure of a deep neural network, we are interested in investigating the possibility of optimizing the shape of the activation function(s)
on a node-by-node basis. We now show how this can be achieved within the context of infinite-dimensional regularization theory.
\subsection{Choice of regularization functional}
For practical relevance, the scheme should favor simple solutions such as an identity or a linear scaling. This will retain
the possibility of performing a classical linear regression. 
It is also crucial that
the activation function $\sigma$ be differentiable to be compatible with the chain rule when the backpropagation algorithm is used to train the network.
Lastly, we want to promote activation functions that are locally linear (such as the ReLU) since these appear to work best in practice. 
\revise{If the two aforementioned constraints are satisfied, then the activation function is CPWL.
As this property is conserved through (multivariate) composition, it implies that the resulting map $\V f_{\rm deep}: \R^{N_0} \to \R^{N_L}$ is CPWL as well, which is highly desirable for applications  \cite[]{Strang2018}.}
 Hence, an idealized solution would be a function $\sigma: \R \to \R$  whose second derivative vanishes almost everywhere. 

As measure of sparsity, we use the ``total-variation'' norm $\|\cdot\|_\Spc M$ associated with the Banach space
\begin{align}
\label{Eq:Mbyduality}
\Spc M(\R)=\{f \in \Spc S'(\R): \|f\|_{\Spc M}\eqdef{ \sup_{\varphi \in \Spc S(\R): \|\varphi\|_{\infty}\le 1}\langle f ,\varphi\rangle}<\infty\}
\end{align}
where $\Spc S'(\R)$ is Schwartz's space of tempered distributions, which is the continuous dual of $\Spc S(\R)$ (the space of smooth and rapidly-decreasing test functions on $\R$). Note that our definition of $\|\c[ot\|_\Spc M$ (by duality) is equivalent to the notion of total-variation used in measure theory  \citep{Rudin1987}. The critical point for us is that the latter is a slight extension of the $L_1$-norm:
The basic property is that $\|f\|_{L_1}\eqdef \int_{\R}|f(x)|\dint x=\|f\|_{\Spc M}$ for any $f \in L_1(\R)$, which implies that
$L_1(\R) \subseteq \Spc M(\R)$. However, the shifted Dirac distribution $\delta(\cdot-x_m) \notin L_1(\R)$ for any shift $x_m\inR$, while $\delta(\cdot-x_m) \in \Spc M(\R)$ with $\|\delta(\cdot-x_m)\|_{\Spc M}=1$, which shows that the space $\Spc M(\R)$ is (slightly) larger than $L_1(\R)$.

To favor neuronal activation functions $\sigma: \R \to \R$ with ``sparse'' second derivatives, we shall therefore impose a bound on their second total-variation, which is defined as
$$
{\rm TV}^{(2)}(\sigma)\eqdef \|\Dop^2 \sigma\|_{\Spc M}={ \sup_{\varphi \in \Spc S(\R): \|\varphi\|_{\infty}\le 1} }
\langle \Dop^2 \sigma ,\varphi\rangle
$$
where $\Op D^2=\frac{\dint^2}{\dint x^2}$ is the second derivative operator.
The connection with ReLU is that $\Dop^2 \{{\rm ReLU}\}=\delta$, which confirms that the ReLU activation function is intrinsically sparse with  ${\rm TV}^{(2)}({\rm ReLU})=1$.

Since our formulation involves a joint optimization of all network components, it is important to decouple the effect of the various stages. The only operation that is common to linear transformations and pointwise nonlinearities is a linear scaling, which is therefore transferable from one level to the next.
Since most regularization schemes are scale-sensitive, it is
essential to prevent such a transfer. We achieve this by restricting the class of admissible weight vectors $\M w_{n,\ell}$ acting on a given node indexed by $({n,\ell})$ to those that have a unit norm. In other words, we shall normalize the scale of all linear modules with the introduction of the new variable 
$\M u_{n,\ell}=\M w_{n,\ell}/\|\M w_{n,\ell}\|$. 

\subsection{Supporting optimality results}
As preparation for our representer theorem, we present a lemma on the ${\rm TV}^{(2)}$-optimality of piecewise-linear interpolation. This enabling result is deduced
from the general spline theory presented in \citep{Unser2017}, as detailed in the appendix\footnote{
As it turns out, the non-obvious part is to actually prove that the required hypotheses 
are met; in particular, the weak* continuity of the dirac functionals in the topology specified by \eqref{Eq:BV2norm}.}. We then provide arguments to disqualify the use of the more conventional  Sobolev-type regularization.

The formal definition of our native space (i.e., the space over which the optimization is performed) is
\begin{align}
\label{Eq:Native}
{\rm BV}^{(2)}(\R)=\{f: \R \to \R: \|\Dop^2 f \|_{\Spc M}<\infty\},
\end{align}
which is the class of functions with bounded second total variation. 
As explained in Appendix \ref{App:BVTopology}, we can endow ${\rm BV}^{(2)}(\R)$ with the norm
\begin{align}
\label{Eq:BV2norm}
\|f\|_{{\rm BV}^{(2)}}\eqdef \|\Dop^2 f \|_{\Spc M} + \revise{\sqrt{|f(0)|^2 + |f(1)-f(0)|^2}},\end{align}
which turns it into a bona fide Banach space.
%\footnote{Our native space is the second-order counterpart of the classical BV space: ${\rm BV}(\R)=\{f: \R \to \R: {\rm TV}(f)=\|\Dop  f \|_{\Spc M}<\infty\}$  (functions of bounded variation).}; i.e., those for which our regularization criterion is well-defined. 
We can then also guarantee that this space %also show that ${\rm BV}^{(2)}(\R)$ 
is large enough---i.e., $\Spc S(\R) \subseteq {\rm BV}^{(2)}(\R) \subseteq \Spc S'(\R)$---to represent any function $f: \R \to \R$ with an arbitrary degree of precision (see explanations after the proof of Theorem 10 in Appendix \ref{App:BVTopology}). The problem of interest is then to search for the optimal interpolant of a series of data points
within that space.

\begin{lemma} [${\rm TV}^{(2)}$-optimality of piecewise-linear interpolants]
\label{Theo:L1spline}
Consider a series of scalar data points $(x_m,y_m), m=1,\dots,M$ with $M>2$ and $x_1\ne x_2$.
Then, under the hypothesis of feasibility (i.e., $y_{m_1}=y_{m_2}$ whenever $x_{m_1}=x_{m_2}$), the extremal points of the interpolation problem
\begin{align*}
%\label{eq:genproblem2}
\arg \min_{{f \in {\rm BV}^{(2)}(\R)}} \| \Dop^2 f\|_{\Spc M} \quad\mbox{ s.t. }\quad  f(x_m)=y_m, m=1,\dots,M
\end{align*}
are nonuniform splines of degree $1$ (a.k.a. piecewise-linear functions) with no more than $(M-2)$ adaptive knots.
\end{lemma}
The proof together with the relevant background in functional analysis is given in the Appendix. The feasibility hypothesis in Lemma \ref{Theo:L1spline} is not restrictive since a function returns a single value for each input point. 
We are aware of two antecedents to Lemma \ref{Theo:L1spline} (e.g., \cite[Corollary 2.2]{Fisher1975}, \cite[Proposition 1]{Mammen1997}); these earlier results, however, are not in the form suitable for 
our purpose because they restrict the domain of $f$ to a finite interval. Our result is also more precise because it yields the full solution set (as the convex hull of the extremal points) and gives a stronger bound on the maximum number of knots.

%Let us denote the number of knots of such a spline by $K$. 
Lemma \ref{Theo:L1spline} implies  that there exists an optimal interpolator, not necessarily unique, whose generic parametric form is given by \begin{align}
\label{Eq:spline}
f_{\rm spline}(x)=b_1 + b_2 x + \sum_{k=1}^K a_k (x-\tau_k)_+
\end{align}
where $(x)_+\eqdef \max(x,0)={\rm ReLU}(x)$, with the 
caveat that the intrinsic spline descriptors, given by the (minimal) number of knots $K$ and the knot locations $\tau_1, \dots,\tau_K \in \R$, are not known beforehand.
This means that these descriptors need to be optimized jointly with the expansion coefficients $\V b=(b_1,b_2)\inR^2$ and $\V a=(a_1,\dots,a_K)\inR^K$. Ultimately, this translates into a solution that has a polygonal graph with breakpoints $f_{\rm spline}(\tau_k), k=1,\dots,K$ and that perfectly interpolates the data points otherwise, as shown in Figure \ref{Fig:Interpol}.
%The non-trivial part here is that the
%knots of the interpolating spline are adaptive---that is, data-dependent---and  lesser than the number of data points.
%
%

%Before proceeding further, let us make briefly comment on the result. 

Since ${\rm TV}^{(2)}$-regularization penalizes 
the variations of the derivative, it will naturally produce (sparse) solutions
with a small number of knots. This means that an optimal spline will typically have fewer knots than there are data points, while the list of its knots $\{\tau_1,\dots,\tau_K\}$ with $K<M$ may not necessarily be a subset of $\{x_1,\dots,x_M\}$, as illustrated in Figure \ref{Fig:Interpol}.
This push towards model simplification (Occam's razor) is highly desirable. It distinguishes this formulation of splines from the more conventional one, which, in the case of interpolation, simply tells us ``to
connect the dots'' with $K=M$ and
$\tau_m=x_m$ for $m=1,\dots,M$ (see %Proposition \ref{Theo:L2spline} and the 
the solid-line illustration in Figure \ref{Fig:Interpol}).

\begin{figure}
\centering
\includegraphics[width=10cm]{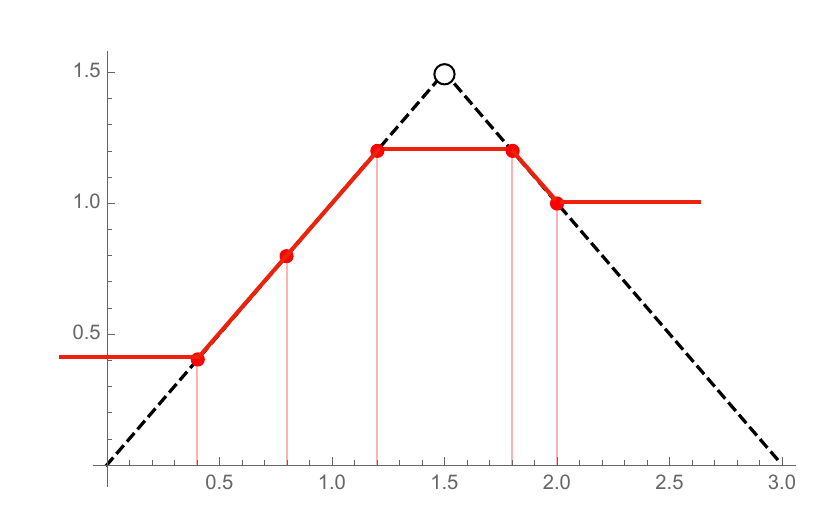}
\caption{\label{Fig:Interpol} \small Conventional (solid line) vs.\ sparse (dashed line) piecewise-linear interpolants.
The 5 data points are shown as red dots; they coincide with the knots of the conventional interpolant.
The sparse solution, by contrast, has a single knot at $\tau_1=1.5$ (circle), an argument value that is found in none of the data points.}
\end{figure}

It is well known that the classical linear interpolator is the solution of the following variational problem, which we like to see as the precursor of RKHS kernel methods  \citep{Prenter:1975,Wahba1990}.

\begin{proposition}[Sobolev optimality of piecewise-linear interpolation]
\label{Theo:L2spline}
Let the native space be the first-order Sobolev space $H^1(\R)=\{f: \R \to \R: \|\Dop f\|^2_{L_2}+|f(0)|^2<\infty\}$.
Given a series of distinct data points $(x_m,y_m), m=1,\dots,M$, the interpolation problem
\begin{align*}
%\label{eq:problem3}
\arg \min_{{f \in H^1(\R)}} 
{\int_\R |\Dop f(x)|^2\dint x} %\nonumber \\
\mbox{  s.t.  } \ f(x_m)=y_m, \ m=1,\dots,M
\end{align*}
has a unique piecewise-linear solution that can be written as
\begin{align}
\label{Eq:spline2}
s_2(x)=b_1 + \sum_{m=1}^M a_m (x-x_m)_+.
\end{align}
\end{proposition}
While the result is elegant and translates into a straightforward implementation, the scheme can be cumbersome for large data sets because the
number of parameters in \eq{Eq:spline2} increases with the number of data points. 
 %In other words, there is no intrinsic mechanism for model simplification.
The other limitation is that the use of
$\|\Dop f\|_{L_2}$-regularization disqualifies the simple linear solution
$f(x)=a x$, which has an infinite cost.

As one may expect, there are also direct extensions of Lemma \ref{Theo:L1spline} and Proposition \ref{Theo:L2spline} for regularized least-squares
approximations. Moreover, the distinction between the two types of solutions---smoothing splines \citep{Schoenberg1964} vs.\ adaptive regression splines \citep{Mammen1997}---is even more striking\footnote{In the least-square setting, one can adjust the strength of ${\rm TV}^{(2)}$-regularization  to control the number of knots and thereby produce solutions with $K \ll M$.} for noisy data fitting applications, which brings us back to our initial goal: the design and training of neural networks.

\subsection{Representer theorem for deep neural networks}
Our aim is to determine the optimal 
activation functions for a deep neural network in a task-dependent fashion. This problem is inherently ill-posed because activations are infinite-dimensional entities while we only have access to finite data. As in the case of interpolation, we resolve the ambiguity by imposing an appropriate form of regularization. Having singled out
${\rm TV}^{(2)}$ as the most favorable choice, we now proceed
with the enunciation of our representer theorem for deep neural networks. 
%We note that the underlying hypothesis $\sigma_{1,\ell},\dots, \sigma_{N_\ell,n} \in {\rm BV}^{(2)}(\R)$ for $n=1,\dots,N$, is hardly a restriction since $ {\rm BV}^{(2)}(\R)$ is rich enough to reproduce any continuous function with an arbitrary degree of accuracy.
\revise{We have purposefully stated the optimization problem in a generic form that is compatible with the current practice in DNN. 
Specifically, the cost function in \eqref{Eq:DeepCost} includes a standard
data term that penalizes data misfit plus a regularization to constrain the values of the linear weights of the network (e.g., $R_\ell( \M U_\ell)=\| \M U_\ell\|^2_{\rm F}$ in the case of the popular weight-decay penalty).
The novelty is the additional optimization over the neuronal activations $\sigma_{n,\ell}$ and the insertion of the 
${\rm TV}^{(2)}$ term 
to regularize their shape.}

\begin{theorem}[${\rm TV}^{(2)}$-optimality of deep spline networks]
\label{Theo:CNNRep}
Let the $L$-layer feedforward neural network $\M f: \R^{N_0} \to \R^{N_L}$ with node descriptor $(N_0,N_1,\dots,N_L)$ take the form
\begin{align}
\label{eq:DeepNet2}
\V x \mapsto \M f(\V x)=\left(\V \sigma_L \circ \V \ell_{L}\circ \V \sigma_{L-1} \circ \cdots %\circ \V \sigma_2 
\circ \V \ell_2 \circ \V \sigma_1 \circ \V \ell_1\right)(\V x),
\end{align}
which is an alternating composition of the normalized linear transformations $\V \ell_\ell: \R^{N_{\ell-1}} \to \R^{N_{\ell}}, \V x \mapsto \M U_\ell \V x$ with linear weights $\M U_\ell=[\M u_{1,\ell}\  \cdots \  \M u_{N_\ell,\ell}]^T \in \R^{N_{\ell}\times N_{\ell-1}}$ such that
$\|\M u_{n,\ell}\|=1$ and the nonlinear activations $\V \sigma_\ell: \R^{N_{\ell}} \to \R^{N_{\ell}}, \V x%=(x_1,\dots,x_{N_\ell}) 
\mapsto \big(\sigma_{1,\ell}(x_1), \dots, \sigma_{N_\ell,\ell}(x_n)\big)$
with  $\sigma_{1,\ell},\dots, \sigma_{N_\ell,\ell} \in {\rm BV}^{(2)}(\R)$. Given a series of data points $(\V x_m,\V y_m)_{m=1}^M$, we then define the training problem
\begin{align}
\label{Eq:DeepCost}
\arg {\min_{(\M U_\ell), (\sigma_{n,\ell} \in {\rm BV}^{(2)}(\R))}}
\left( \sum_{m=1}^M E\big(\V y_m,\M f(\V x_m)\big) \right. %\nonumber \\
\left. + \mu \sum_{\ell=1}^N R_\ell( \M U_\ell)+ \lambda \sum_{\ell=1,}^L \sum_{n=1}^{N_\ell} {\rm TV}^{(2)}(\sigma_{n,\ell})\right)
\end{align}
where $E: \R^{N_L} \times \R^{N_L} \to \R^+$ is an arbitrary convex error function such that $E(\V y,\V y)=0$ for any $\V y \inR^{N_\ell}$,
$R_\ell: \R^{N_\ell \times N_\ell}\to \R^+$ is some arbitrary convex cost that favors certain types of linear transformations, and $\lambda, \mu\inR^+$ are two adjustable regularization parameters.
If the solution of \eq{Eq:DeepCost} exists, then it is achieved by a deep spline network
with individual activations of the form 
\begin{align}
\label{Eq:splinenetwork}
\sigma_{n,\ell}(x)=b_{1,n,\ell} + b_{2,n,\ell}x + \sum_{k=1}^{K_{n,\ell}} a_{k,n,\ell} (x - \tau_{k,n,\ell})_+,
\end{align}
with adaptive parameters
$K_{n,\ell}\le M-2$, $\tau_{1,n,\ell}, \dots,\tau_{K_{n,\ell},n,\ell} \inR$, and $b_{1,n,\ell}, b_{2,n,\ell}, a_{1,n,\ell}$, $\dots, a_{K_{n,\ell},n,\ell} \inR$.
\end{theorem}

\begin{proof} %[Proof of Theorem \ref{Theo:CNNRep}]
Let the function $\tilde{\M f}: \R^{N_0}\to \R^{N_L}$ be a (not necessarily unique) solution of the problem summarized by \eq{Eq:DeepCost}. This solution is described by \eq{eq:DeepNet2} with some optimal choice of transformation matrices $\tilde{\M U}_{\ell}$ and pointwise nonlinearities
$\tilde{\sigma}_{n,\ell}: \R \to \R$ for $\ell=1,\dots,L$ and $n=1,\dots,N_\ell$. 

As we apply $\tilde{\M f}$ to the data point $\V x=\V x_m$ and progressively move through the layers of the network, we generate a series of vectors
$\V z_{m,\ell} \inR^{N_\ell}$, according to the following recursive definition:

\begin{itemize}
\item Initialization (input of the network):
$\tilde{\V y}_{m,0}=\V x_m$.

\item Recursive update: For $\ell=1,\dots,L$,
calculate
\begin{align}
\label{Eq:auxz}
\V z_{m,\ell}=(z_{1,m,\ell}, \dots, z_{N_{\ell},m,\ell})= \tilde{\M U}_{\ell\,} \tilde{\V y}_{m,\ell-1}
\end{align}
and construct $\tilde{\V y}_{m,\ell}=(\tilde y_{1,m,\ell},\dots,\tilde y_{N_\ell,m,\ell}) \inR^{N_\ell}$ 
with
\begin{align}
\label{Eq:auxvalues}
\tilde{y}_{n,m,\ell}=\tilde{\sigma}_{n,\ell}(z_{n,m,\ell})\quad n=1,\dots,N_\ell.
\end{align}

\end{itemize}
At the output level, we get $\tilde{\M f}(\V x_m)=\tilde{\V y}_{m,L}$ %=(\tilde y_{1,m,L},\dots,\tilde y_{N_\ell,m,L})$ 
for $m=1,\dots,M$, which are the values that 
determine the data-fidelity part of the criterion associated with the optimal network and represented by the term $\sum_{m=1}^M E\big(\V y_m,\M f(\V x_m)\big)$ in 
\eq{Eq:DeepCost}. Likewise, the specification of the optimal linear transforms $\tilde{\M U}_1, \dots, \tilde{\M U}_L$  fixes the regularization cost $\sum_{\ell=1}^L R_\ell(\M U_\ell)$. 
Having set these quantities, we concentrate on the final element of the problem: the characterization of the ``optimal''  activations $\tilde{\sigma}_{n,\ell}: \R \to \R$ in-between the locations $z_{n,m,\ell}$ associated with the ``auxiliary'' data points $\tilde y_{n,m,\ell}=\tilde{\sigma}_{n,\ell}(z_{n,m,\ell})$, $m=1,\dots,M$. The key is to recognize that we can now consider the various activation functions individually because the variation of
$\tilde{\sigma}_{n,\ell}$ in-between data points is entirely controlled by ${\rm TV}^{(2)}(\tilde \sigma_{n,\ell})$ without any influence on the other terms of the cost functional.
Since the solution $\tilde{\M f}$ achieves the global optimum, %and all elements that the enter the determination of the cost functional are fixed, for each node value $(j,n)$
we therefore have that
\begin{align*}
\tilde \sigma_{n,\ell}=\arg \min_{f \in {\rm BV}^{(2)}(\R)} \| \Dop^2 f\|_{\Spc M}\quad\mbox{   s.t.   } \quad f(z_{n,m,\ell})=\tilde y_{n,m,\ell}, \ \ m=1,\dots,M,
\end{align*}
where the ``auxiliary'' data pairs $(z_{n,m,\ell}, \tilde y_{n,m,\ell})$ are specified by \eq{Eq:auxvalues}. After this reformulation, we can apply Lemma \ref{Theo:L1spline}, which proves that, at each node $(n,\ell)$, the minimum is achieved by a nonuniform spline with a number  $K_{n,\ell}$ of knots smaller than the number of data points.

Since the hypothesis of feasibility is implicit in the construction, there is only one case not covered by Lemma \ref{Theo:L1spline}: the singular scenario where all the auxiliary data points associated to a node are equal.
Fortunately, this does not break the argument because such a configuration calls for a (zero-cost) solution of the form $b_1 + b_2 x$ (which is a special case of \eq{Eq:spline} with $K=0$), except for the twist that there are now infinitely many possibilities with $b_1+b_2 z_1=\tilde y_1$.

\end{proof}

%The proof is given in SM-3.
This result translates into a computational structure where each node of the network \big(with fixed index $(n, \ell)$\big) is characterized by\begin{itemize}
\item its number $0\le K=K_{n,\ell}$ of knots (ideally, much smaller than $M$);
\item the location $\{\tau_k=\tau_{k,n,\ell}\}_{k=1}^{K_{n,\ell}}$ of these knots (equivalent to ReLU biases);
\item the expansion coefficients $b_{1,n,\ell},b_{2,n,\ell}, a_{1,n,\ell}, \ldots, a_{K,n,\ell}$,
also written as $\M b=(b_1,b_2) \inR^2$ and $\M a=(a_1, \ldots, a_{K}) \inR^{K}$ to avoid notational overload.
\end{itemize}
The fundamental point is that these parameters (including the number of knots) are data-dependent and adjusted automatically through the minimization
of \eq{Eq:DeepCost}. All this takes place during training.

\section{Interpretation and discussion}
Theorem \ref{Theo:CNNRep} tells us that we can configure a neural network optimally by restricting our attention to piecewise-linear activation functions $\sigma_{n,\ell}$, %. Accordingly, we shall refer to the functions described by \eq{Eq:splinenetwork} %Since these underlying functions are non-uniform splines of degree $1$, they will be called 
or {\em spline activations}, for short.  In effect, this means that the ``infinite-dimensional'' minimization problem specified by \eq{Eq:DeepCost} can be converted into a %(hopefully) 
tractable finite-dimensional problem 
where, for each node $(n,\ell)$,  the parameters to be optimized are the number $K_{n,\ell}$ of knots, the locations $\{\tau_{k,n,\ell}\}_{k=1}^{K_{n,\ell}}$ of the spline knots, and the linear weights $b_{1,n,\ell} , b_{2,n,\ell}, a_{1,n,\ell}, \dots,$ $a_{K_{n,\ell},n,\ell} \inR$.
The enabling property is going to be \eq{Eq:TV2norm1}, which converts the continuous-domain regularization into a discrete $\ell_1$-norm.
This is consistent with the expectation that bounding the second-order total-variation 
favors solutions with sparse second derivatives---i.e., linear splines with the fewest possible number of knots. % $K_{n,\ell}\ll M$---
The idea is that $\ell_1$-minimization helps reducing the number of active coefficients $a_{k,n,\ell}$ \citep{Donoho2006b, Unser2016}. 

The other important feature is that the knots are adaptive and that they can be learned
during training using the standard backpropagation algorithm.  What is required is the derivative of the activation functions. It is given by
\begin{align}
\label{Eq:SplineD}
\sigma'_{n,\ell}(x)=b_{2,n,\ell} + \sum_{k=1}^{K_{n,\ell}} a_{k,n,\ell} \One_{[\tau_{k,n,\ell},+\infty)}(x),
\end{align}
where $\One_{[\tau,+\infty)}(x)$ is an indicator function that is zero for $x<\tau$ and $1$ otherwise. These derivatives are piecewise-constant splines with jumps of height $a_{k,n,\ell}$ at the knot locations 
$\tau_{k,n,\ell}$. By differentiating \eq{Eq:SplineD} once more, we get
\begin{align}
\label{Eq:SplineD2}
\sigma''_{n,\ell}(x)=\sum_{k=1}^{K_{n,\ell}} a_{k,n,\ell} \delta(x-\tau_{k,n,\ell}),
\end{align}
where $\delta$ is the Dirac distribution. Owing to the property that $\|\delta(\cdot-\tau_{k,n,\ell})\|_{\Spc M}=1$, we then readily deduce that
\begin{align}
\label{Eq:TV2norm1}
{\rm TV}^{(2)}\{\sigma_{n,\ell}\}=\|\sigma''_{n,\ell}\|_{\Spc M}=\sum_{k=1}^{K_{n,\ell}} | a_{k,n,\ell}|=\|\M a_{n,\ell}\|_1,
\end{align}
which converts the continuous-domain regularization into a more familiar minimum $\ell_1$-norm constraint on the underlying expansion coefficients.

\subsection{Link with existing techniques}
What is even more interesting, from a practical point of view, is that the corresponding system translates into a deep ReLU network modulo a slight modification of the standard architecture described by \eq{eq:DeepNet}. Indeed, the primary basis functions in \eq{Eq:splinenetwork} are shifted ReLUs, so that each spline activation $\sigma_{n,\ell}$ can be realized by way of a simple one-layer ReLU subnetwork with the spline knots being encoded in the biases. In particular, when the only active coefficients is
$a_{n,\ell}\eqdef a_{1,n,\ell} $ (i.e., $b_{1,n,\ell}=0 $, $b_{2,n,\ell}=0$, and $K_{n,\ell}=1$), we have a perfect equivalence with the classical deep ReLU structure described by \eq{eq:DeepNet} with $\sigma_{n,\ell}(x)=(x)_+$.
The enabling property is that
$$
(\M w^T_{n,\ell}\V x - z_{n,\ell})_+=(a_{n,\ell} \M u^T_{n,\ell}\V x - z_{n,\ell})_+=a_{n,\ell}( \M u^T_{n,\ell}\V x - \tau_{n,\ell})_+,
$$
with $\M u_{n,\ell}=\M w_{n,\ell}/\|\M w_{n,\ell}\|$, $a_{n,\ell}=\|\M w_{n,\ell}\|$ and $\tau_{n,\ell}= z_{n,\ell}/a_{n,\ell}$. Concretely, this means that, for every layer $\ell$, we can absorb the single ReLU coefficients $a_{n,\ell}, n=1,\dots,N_\ell$ into the prior linear transformation and consider unnormalized transformations $\M W_\ell=[\M w_{1,\ell} \ \dots \ \M w_{N_\ell},]^T$ \big(as in \eq{eq:CPWL}\big)
rather than the normalized ones of Theorem \ref{Theo:CNNRep} with $\M u_{n,\ell}=\M w_{n,\ell}/\|\M w_{n,\ell}\|$.

Theorem \ref{Theo:CNNRep} then suggests that the next step in complexity is to add the linear term $b_{1,n,\ell}+b_{2,n,\ell}x$ to each node, since its
regularization cost vanishes. \revise{Interestingly, the suggested configuration---that is, one ReLU plus an adjustable linear term per neuron---is equivalent to the parametric ReLU model (PReLU)
of  \cite{He2015}, which has been found to systematically outperform the baseline ReLU configuration in real-world applications. }
The other design extreme is to let $\lambda \to \infty$, in which case the whole network collapses, leading to an affine mapping of the form
$\M f(\V x)=\M W \V x - \M b$ with $\M W \inR^{N_L \times N_0}$ and $\M b \inR^{N_L}$.
More generally, the framework provides us with the possibility of controlling the number of knots (and hence the complexity of the network) through the simple adjustment of the regularization parameter 
$\lambda$, with the number of knots increasing as $\lambda\to0$.

Among the various attempts in the literature to optimize the shape of the activation functions in deep neural networks, there \revise{ is one} scheme that is remarkably close to the \revise{optimal} solution suggested by our theorem: the APL (adaptive piecewise-linear activation) framework of \cite{Agostinelli2015} 
  in which each neuron is represented as a linear combination of shifted ReLUs, with the parameter being determined during training. The only difference is that their number of ReLUs is fixed a priori and that their model does not include the linear term $b_1+b_2x$. While Agostinelli et al.'s formulation does not involve any explicit regularization, they found in their experiments that is was helpful to add some mild $\ell_2$ penalty on the ReLU coefficients (to be contrasted with the sparsity-promoting $\ell_1$-penalty that results from our theorem) to avoid numerical instability.
The good news in support of our theorem is that they report substantial improvement (9.4\% and 7.5\% relative error decrease, respectively) on state-of-the-art CNN (with fixed RELU activations) on the CIFAR-10 and CIFAR-100 classification benchmarks.

A characteristic property of deep spline networks, to be considered here as a superset of the traditional deep ReLU networks, is that they 
produce an input-output relation that is continuous and piecewise-linear (CPWL) in the following sense: the corresponding function $\M f$ is continuous $\R^{N_0}\to \R^{N_L}$;
its domain
$\R^{N_0}=\bigcup_{k=1}^K P_k$ can be partitioned into a finite set of non-overlapping convex polytopes $P_k$ over which it is %(locally) 
affine \citep{Tarela1999, Wang2005}. More precisely,
$\M f(\V x)=\V f_k(\V x)$ for all $\V x \in P_k$ where $\V f_k: \R^{N_0}\to \R^{N_L}$ has the same parametric form as in \eq{eq:CPWL} .
This simply follows from the observation that $\V \sigma_{\ell}=(\sigma_{1,\ell},\dots,\sigma_{N_\ell,\ell})$, with $\sigma_{n,\ell}$ as specified by \eq{Eq:splinenetwork}, is CPWL and that the CPWL property is conserved through functional composition. In fact, the CPWL property for $N=1$ is equivalent to the function %$f: \R\to\R$ 
being %expressible in the form given by \eq{}; that is, to $f$ being 
a nonuniform spline of degree 1.
% Another of expressing this equivalence is that $f(x)=b+a x$ with $b=$  for all $x \in P_k=[\tau_{k-1},\tau_k]$. In other words, there is an equivalence between CPWL and the property of being a a non-uniform spline and CPWL.

Another powerful architecture that is known to generate CPWL functions is the MaxOut network  \citep{Goodfellow2013}. There, the non-linear steps
$\V \sigma_\ell$ in \eq{eq:DeepNet} are replaced by max-pooling operations. It turns out that these operations are also expressible in terms of deep splines, as illustrated
in Figure \ref{Fig:Hinge} for the simple case where the maximum is taken over two inputs. Interestingly, this conversion requires the use of the linear term
which is absent in conventional ReLU networks. This reinforces the argument made by Goodfellow {\em et al.} concerning the capability of MaxOut to learn activation functions.

\begin{figure}
\centering
\includegraphics[width=15cm]{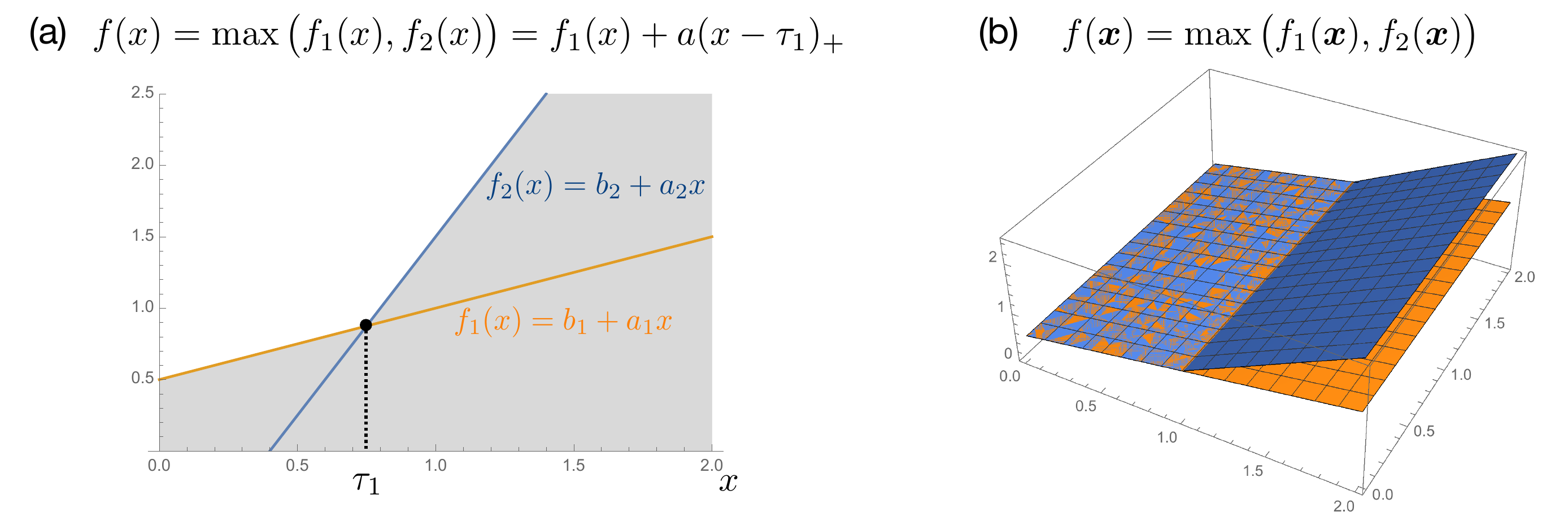}
\caption{\label{Fig:Splines} \small Deep-spline implementation of MaxOut with $N=2$. (a) In one dimension, the spline parameters are 
$a=(a_2-a_1)$ and $\tau_1=(b_2-b_1)/(a_2-a_1)$. (b) In two or more dimensions, $f(\V x)=b_1 + \V a_1^T \V x+ a (\V u^T \V x -\tau_1)_+$, where $\V u$ is the unit vector perpendicular to the hinge. \label{Fig:Hinge}}
\end{figure}

An attractive feature that is offered by the deep-spline parameterization is the possibility of suppressing a network layer---or rather, merging two adjacent ones---when the optimal solution is such that $K_{n,\ell}=0$ for $\ell$ fixed and $n=1,\dots,N_\ell$. This is a property that results from the presence of the linear component and 
%$b_1 + b_2 x$, 
has not been exploited so far.

\subsection{Generalizations}
The optimality result in Theorem \ref{Theo:CNNRep} holds for a remarkably broad family of cost functions, which should cover all cases of practical interest.
The first condition is that the data term, as its name suggest, be solely dependent on $\V y_m$ and $\V f(\V x_m)$. The second is that the regularization of the weights---the part that constrains the linear steps---and the regularization of the individual activation functions are decoupled from each others.
Obvious generalizations of the result include
\begin{itemize}
\item cases where the optimization in \eq{Eq:DeepCost} is performed over a subset of the components while
other network elements such as critical linear weights, activation functions\footnote{A prominent example is the use of the softmax function \citep{Bishop2006,Goodfellow2016} to convert the output of a neural network into a set of pseudo-probabilities.} or even pooling operators are fixed beforehand; in particular, this includes the important subclass of networks that are not fully connected;
\item configurations such as those found in convolutional networks where some (tunable) activation functions are shared among multiple nodes;
\item generalized forms of regularization
where ${\rm TV}^{(2)}(\sigma_{n,\ell})$ is substituted by $\psi\big({\rm TV}^{(2)}(\sigma_{n,\ell})\big)$, where $\psi: \R^+ \to \R^+$ is any monotonically increasing function.
\end{itemize}
While the first and second scenarios require a slight reformulation of the optimization problem, it is still possible to invoke the same kind of ``interpolation'' argument as in the proof of Theorem \ref{Theo:CNNRep}. The third generalization is obvious since the (constrained) miminization of ${\rm TV}^{(2)}(\sigma_{n,\ell})$ is equivalent to the minimization of $\psi\big({\rm TV}^{(2)}(\sigma_{n,\ell})\big)$.

\revise{The statement in Theorem \ref{Theo:CNNRep} refers to the global optimum of \eqref{Eq:DeepCost}, which is often hard to reach in practice (because the underlying problem is highly non-convex). It turns out that the argument of the proof is also applicable to local minima and/or saddle points of the cost functional.}

\revise{
By relying on the supporting mathematics in Appendix B for general spline-admissible operators $\Lop$, 
it is possible to revisit the proof of Theorem \ref{Theo:CNNRep} to determine
the parametric form of the optimal activations for higher-order versions of TV regularization; i.e., ${\rm TV}^{(n)}(\sigma)=\|\Dop^n \sigma\|_{\Spc M}$. 
This yields optimal activations that are non-uniform polynomial splines of degree $n>2$. 
While such solutions have a higher-order of differentiability, they are less favourable globally because the underlying spline property is not retained through composition, meaning that the larger the number of layers,
the larger the polynomal degree of the ``polytopes'' of the resulting network.
By contrast, the CPWL property of the linear splines in \eqref{Eq:splinenetwork} is preserved through composition, so that the resulting deep spline DNN
can be also be interpreted as a flat (or shallow) multidimensional piecewise-linear spline.
The other way of inducing CPWL activations
is through the quadratic Sobolev 1 regularization of Proposition \ref{Theo:L2spline}.
However, this solution has two shortcommings: (i) its inability to represent the identity which would result in an infinite cost, and (ii) its lack of sparsity.% since the number of knots is necessarily equal to $M$.
}

\subsection{Comparison with kernel methods}

We like to contrast the result in Theorem \ref{Theo:CNNRep} with the classical representer theorem of machine learning \citep{Scholkopf2001}.
The commonality is that both theorems provide a parametric representation of the solution in the form of a  linear ``kernel'' expansion.
The primary distinction is that the classical representer theorem is restricted to ``shallow'' networks with $L=1$.
Yet, there is another difference even more crucial for our purpose: the fact that the knots $\tau_k$ in \eq{Eq:splinenetwork} are {\em adaptive} and few ($K \ll M)$, while the centers $\V x_m$ in \eq{Eq:Kernelf} are {\em fixed} and as numerous as there are data points in the training set.
In addition, the ReLU function $(x)_+$ is not a %valid 
kernel in the traditional sense of the term because it is not positive-definite. We note, however, that it can be substituted by another equivalent spline generator $|x|$, which is conditionally positive-definite \citep{Micchelli1986, Wendland2005}.
Again, the property that makes this feasible is the presence of the linear term $b_{1,n,\ell} + b_{2,n,\ell}x$.

There is also a conceptual similarity between the result of Theorem \ref{Theo:CNNRep} and a recent representer theorem for deep kernel networks \citep{Bohn2018}, which results in a solution that is a composition of $L$ multi-valued kernel estimators of the classical RKHS form given by \eqref{Eq:Kernelf}. 
Again, the two main differences with the present framework are: (i) each layer of the deep kernel network is a multivariate non-linear map, which does not necessarily allow for affine transformations (e.g. linear regressions) and, (ii) the kernel expansion in each layer requires as many basis functions as there are
training data; this amounts to a total of $L\times M$ linear parameters---this can rapidly become prohibitive, not to mention the complexity of the underlying (non-convex) optimization task.
The first shortcoming can easily be fixed by inserting intermediate affine transformations in direct analogy with the type of architecture covered by Theorem \ref{Theo:CNNRep}. The second limitation is more fundamental and can probably only be removed by adopting some kind of generalized TV regularization in the spirit of \cite{Unser2017}; in short, this calls for an extension of Theorem \ref{Theo:CNNRep} for multivariate activations, which is currently work in progress.
%\revise{As preliminary step, we can already report on a multivariate generalization of Lemma 1, which yields sparse (multi)-kernel expansions that
%are also relevant to machine learning and interesting on their own right.}
\revise{ \subsection{Towards a practical implementation}}

\revise{ While the solution of Theorem \ref{Theo:CNNRep} is conceptually appealing, it can be expected to be harder to implement than
fixed kernel/RKHS methods since the optimization is not only over the linear weights $\V a_{n,\ell}$ and $\V b_{n,\ell}$, but also over the number and positions of the corresponding spline knots.
There is also always a risk that an increase in the number of degrees of freedom may compromise the generalization ability of the result network, which means that the method will need to be carefully tested
and validated on real data.
A possible strategy for making the optimization easier is to constrain the ReLU units to lie on a grid---in the spirit of \cite{Gupta2018}---and to then rely on standard iterative $\ell_1$-norm minimization techniques to produce a sparse solution \citep{Donoho2006b, Foucart2013,Unser2016}. 
%Since the $\ell_1$ optimization is convex at the level of a single layer, it would put the scheme on the same footing as the RKHS method of \citep{Bohn2018}.
% to remove unnecessary knots.
%A possible \revise{strategy} is to allocate a knot budget per node and to then rely on standard iterative $\ell_1$-norm minimization techniques to produce a sparse solution \citep{Donoho2006b, Foucart2013,Unser2016}. 
Such a scheme may still require some explicit knot-deletion step, either as post-processing or during the training iterations, to effectively trim down the number of parameters. A potential difficulty is that
minimum ${\rm TV}^{(2)}$ interpolants are typically non-unique, because the underlying regularization is semi-convex. This means that
the solution found by an iterative algorithm---assuming that the minimum of the regularization energy is achieved---is not necessarily the sparsest one within the (convex) solution set.
Designing an algorithm that can effectively deal with this issue will be a very valuable contribution to the field.\\
} 

%\section{Directions for future research}
\section{Conclusion}
The main contribution of this work is to provide the theoretical foundations for an integrated approach to neural networks
where a sub-part of the design---the optimal shaping of activations---can be transferred to the training part of the process
and formulated as a global optimization problem. 
It also deepens the connection between splines and multilayer ReLU networks, as a pleasing side product.
While the concept seems promising and includes the theoretical possibility of suppressing unnecessary layers, it %also 
raises a number of issues that can only be answered through extensive experimentation with real data.
\revise{There are already strong indications in the literature (e.g., the improved performance of PReLU)
of the practical usefulness of the linear-activation component that is suggested by the theory and %typically 
not present %(at least not in the systematic form of Theorem \ref{Theo:CNNRep})
in traditional ReLU systems. } 
The next task ahead is to demonstrate the capability of \revise{more complex} spline activations to improve upon the state-of-the-art.
(Except for a potential risk of over-parameterization, deep-spline networks should perform at least as well as deep ReLU, PReLU or APL networks since the latter constitute a subset of the former.)

% Inspired by the many successes of convolutional networks for the classification and processing of signals, it makes good sense to investigate deep-spline variants of such systems where some of the weighting matrices and spline activations are shared in a sliding-window fashion. There is already evidence that the performance of such networks can be boosted by tuning the activation function to the corresponding convolution kernel, with a relatively modest parametric and computational overhead \citep{Agostinelli2015}.
 
\revise{We expect the greatest challenge for training a deep spline network to be the proper optimization of the number of knots at each neuron, given that
the solution with the fewest parameters is the most desirable.}
\revise{ In short, we are still in need of a practical and efficient solution for training a deep neural network with fully adaptable activations that globally produces a continuous and piecewise-linear input-output relation; in other words, a DNN that implements an adaptive multidimensional linear spline.}
%
%\begin{itemize}
%\item
%\item 
%\end{itemize}
%
%Ideas and conjectures.
%\begin{itemize}
%\item Because deep ReLU networks produce piecewise-linear functions which have universal approximation properties, we suspect that they will generally perform better than deep networks with other (say, sigmoidal) activation functions. This need to be tested out for specific architectures (say, flat and deep for $d=1$ and $d=2$) by comparing the residual errors after training. 
%\item 
%In line with this analysis, we suspect that the trickiest part of the process is the learning of the knots (or bias values for the ReLU), since all the rest is pretty much linear.
%\item Inspired by the many successes of convolutional networks for the classification and processing of signals, we propose a generalized version of such systems where some of the weighting matrices and generalized activation functions are shared in a sliding window fashion. The hope is that one may obtain a boost in performance by tuning the activation function to the corresponding convolution kernel, with a relatively modest parametric and computational overhead. 
%\end{itemize}

%\acknow{
%The research was partially supported by the Swiss National Science
%Foundation under Grant 200020-162343.
%% and the European Commission under Grant ERC-2010-AdG 267439-FUN-SP.
%The author is thankful to Shayan Aziznejad, Anais Badoual and Kyong Hwan Jin for helpful discussions.}
%
%\showacknow

%%%%%%%%%%%%%%%%
%%%%%%%%%%%%%%%%
%%%%%%%%%%%%%%%%

%\newpage
\appendix
%\onecolumn
%%\large
%\begin{center} {\huge \bf Supporting material (SM)} \end{center}
\section*{Appendices}
The proof of Lemma 1 is based on some foundational results in \citep{Unser2017} that rely on compactness arguments requiring the weak* topology.
We therefore start with a brief review of the relevant notions from functional analysis (Appendix \ref{App:Background}).
We then specify the topology of ${\rm BV}^{(2)}(\R)$  in Appendix \ref{App:BVTopology} and precisely delineate its predual space in Theorem \ref{Theo:Predual}. This latter characterization is the key to the proof of Lemma \ref{Theo:L1spline} that is presented in Appendix \ref{App:ProofLemma}. 
%Finally, we conclude with the derivation of our representer theorem in SM-3.

%\section*{SM-0. Background on continuity and weak* continuity}
\section{Background on continuity and weak* continuity}
\label{App:Background}
\begin{definition} Let $v: u \mapsto \langle v,u\rangle$ be a linear functional on a Banach space $\Spc U$ equipped with the norm $\|\cdot\|_{\Spc U}$. Then, $v:  \Spc U \to \R$ is said to be continuous if $\ \lim_{n\to\infty}\langle v, u_n\rangle=\langle v,u\rangle$ for any sequence $(u_n)$ in 
$\Spc U$ such that $\lim_{n\to\infty}\|u_n-u\|_{\Spc U}=0$. 
\end{definition}

We recall that $\Spc U'$ (the continuous dual of $\Spc U$) is the vector space that is formed of the linear functionals that are continuous on $\Spc U$;
it is a Banach space equipped with the dual norm
$$
\| v \|_{\Spc U'}= \sup_{\varphi \in \Spc U: \|u\|_{\Spc X}\le 1} \langle v, u\rangle.
$$
By following up on the above property, one specifies the space of linear functionals that are continuous on $\Spc U'$, which yields the Banach space $\Spc U''$.
A standard result in functional analysis is that $\Spc U$ is continuously embedded in its bidual $\Spc U''$, which is indicated as $\Spc U\embedC \Spc U''$, with the two spaces
being isometrically isomorphic (i.e., $\Spc U=\Spc U''$) if and only if $\Spc U$ is reflexive
 \citep{Rudin1991}. In other words, the construction of the bidual $\Spc U''$ gets us back to the initial space in the reflexive case only. 
 
A primary case of interest for this paper is $\ \Spc U'=\Spc M(\R)$ which is not reflexive.
For such a scenario, the proper way to deduce the {\em predual} space $\Spc U$ from $\Spc V=\Spc U'$ is through the identification of the
 linear functionals that are weak*-continuous on $\Spc V$.
 
%In other words, the construction of the bidual $\Spc U''$ gets us back to the initial space in the reflexive case only. To achieve
%the same in the non-reflexive scenario, we need to invoke
%%reformulate the underlying notion of convergence; that is, replace the norm (or strong) topology $(\Spc U',\|\cdot\|_\Spc U')$ by 
%the weak* topology. %, which is weaker than the weak topology.

%In fact, three different types of convergence can be defined on the dual space $\Spc U'$.
\begin{definition}[Weak* topology]
A sequence $(v_n)_{n=1}^\infty$ in
$\Spc V=\Spc U'$ is said to converge to $v$ in the weak* topology if $\lim_{n\to\infty}\langle v_n-v, u\rangle=0\ $ for all $u \in \Spc U$.
\end{definition}
\begin{definition}[weak* continuity]
A linear functional $u:  \Spc U' \to \R$ is said to be weak*-continuous if $\ \lim_{n\to\infty}\langle u, v_n\rangle=\langle u,v\rangle$ for any sequence $(v_n)$ that converges to $v$ in the weak* topology. 
\end{definition}

\begin{proposition}[{\rm see {\citep[Theorem IV.20, p. 114]{Reed1980}}}]
\label{Prop:WeakStar}
The only weak* continuous linear functionals on $\Spc U'$ are the elements of $\Spc U$.
\end{proposition}
%, with the two spaces
%being isometrically isomorphic (i.e., $\Spc X=\Spc X''$) if and only if $\Spc X$ is reflexive
%Note that, despite the qualifier ``weak'', the functional property of weak* continuity is actually {\em stricter} than continuity
%%; that is,
%%$$
%%\nu: \Spc X\to \R \mbox{ is  weak*-continuous} \quad \Rightarrow \quad \nu: \Spc X\to \R \mbox{ is  continuous} 
%%$$
%with the two concepts being equivalent iff. $\Spc X$ is reflexive. Indeed, the set of linear functionals that are continuous
%on $\Spc X'$ is given by $\Spc X''\supseteq \Spc X$

The main point is that, despite the qualifier ``weak'', the functional property of weak* continuity is actually {\em stricter} than continuity.

In practice, it is relatively straightforward to establish the continuity of $u:\Spc V \to \R$ since the property is equivalent to the existence of a constant $C>0$ such
$$
| \langle u, v\rangle| \le C \|v\|_{\Spc V}
$$
for all $v\in \Spc V$, which also yields $\|u\|_{\Spc V'}\le C<\infty$.
By contrast, proving that $v: \Spc V \to \R$ is weak*-continuous in the non-reflexive scenario requires the precise characterization of the predual of $\Spc V$, which is typically more demanding mathematically. For instance, the property that $\Spc M(\R)=\big(C_0(\R)\big)'$ is a fundamental result  in measure theory
known as the Riesz-Markov theorem \citep{Rudin1987}. 

%The latter criterion is just a restatement of the definition since $\Spc X''=$ (the double dual of $\Spc X$).

For example, the functionals 
$\varphi\mapsto\langle 1,\varphi\rangle$ and $\varphi\mapsto\langle \Indic_{[0,1]},\varphi\rangle$ are continuous on $\Spc M(\R)=\big(C_0(\R)\big)'$ because
the ``generalized'' functions $1$ and $ \Indic_{[0,1]}$ are bounded in the $\sup$-norm. However, they both fail to be weak*-continuous; i.e., 
$1 \notin C_0(\R)$ because it does not decay at infinity, and  $ \Indic_{[0,1]} \notin C_0(\R)$ because it is not continuous everywhere.
In the latter example, we may recover weak* continuity by considering a smoothed version of the indicator function.

These considerations are central to the proof of Lemma 1 because it requires the weak* continuity of the 
sampling functional $\delta(\cdot-x_m)$. While the sampling operation is continuous on ${\rm BV}^{(2)}(\R)$, it is not necessarily weak*-continuous; at least not in the canonical topology that is proposed in \cite[p. 780]{Unser2017} (e.g., polynomial spline example with $N_0=2$ and $\V \phi=(\delta,-\delta')$).
This is the reason why we need to revisit the construction of our native space, as detailed in Appendix \ref{App:BVTopology}, and establish a new operational criterion for testing weak* continuity (%see \eqref{Eq:weakcont} in 
Theorem \ref{Theo:Predual}). 
%We shall also identify its predual (Theorem 4) in order to be able to apply the .

%\section*{SM-1. Banach structure of ${\rm BV}^{(2)}(\R)$ and of its predual space}
\section{Banach structure of ${\rm BV}^{(2)}(\R)$ and of its predual space}
\label{App:BVTopology}
While the definition of ${\rm BV}^{(2)}(\R)$ given in \eq{Eq:Native} is convenient for expository purposes, it is not directly usable for mathematical analysis because the functional $\|\Dop^2 f \|_{\Spc M}$ is only a semi-norm. To lift the ambiguity due to the non-trivial null space, we
select a biorthogonal system $(\V \phi, \V p)$ for $\Spc N_{\Dop^2}$ (the null space of $\Dop^2$). 
In order to fix the problem of weak* continuity (see explanations surounding Figure \ref{Fig:kernel}), our proposed modification of the canonical scheme
%Because we are dealing with an interpolation problem, our proposed choice\footnote{The canonical choice of analysis functionals $\V \phi=(\delta,-\delta')$ is not appropriate here because of the requirement that the sampling functionals should be weak*-continuous (see proof of Proposition \ref{Prop:samplingContinuous}). } 
is $\V \phi=(\phi_1,\phi_2)=\big(\delta,-\delta+\delta(\cdot-1)\big)$ and 
$\V p=(p_1, p_2)$, with $p_1(x)=1$ and $p_2(x)=x$, which are such that
$\langle p_1,\phi_1\rangle=p_1(0)=1$, $\langle p_1,\phi_2\rangle=-p_1(0)+p_1(1)=0$, $\langle p_2,\phi_1\rangle=p_2(0)=0$ and $\langle p_2,\phi_2\rangle=-p_2(0)+p_2(1)=1$ (biorthogonality property). 
We then rely on \cite[Theorem 5]{Unser2017} to get the following characterization. 

\begin{proposition}[Banach structure of ${\rm BV}^{(2)}(\R)$]
\label{Prop:BanachNative}
Let $(\V \phi, \V p)$ be a biorthogonal system for $\Spc N_{\Dop^2}={\rm span}\{1,x\}$. Then, ${\rm BV}^{(2)}(\R)$ equipped with the norm
$$
\|f\|=\|\Dop^2 f\|_{\Spc M} + %\sum_{n=1}^2 |\langle \phi_n, f\rangle|,
\revise{ \sqrt{|\langle \phi_1, f\rangle|^2 +  |\langle \phi_2, f\rangle|^2}},
$$
is a (non-reflexive) Banach space. Moreover, every $f \in {\rm BV}^{(2)}(\R)$
has the unique direct-sum decomposition
\begin{align}
\label{Eq:directsum}
f=\Op G_\V \phi\{ w\} +p,
\end{align}
where $w=\Dop^2 f \in \Spc M(\R)$, $p=\sum_{n=1}^{2} \langle f, \phi_n \rangle p_n  \in \Spc N_{\Dop^2}$,
%This norm may also be written as $\|f\|=\|w\|_\Spc M + \|p\|_{\Spc N}$ with 
%$\|p\|_{\Spc N}=\sum_{n=1}^{2}  |\langle p, \phi_n \rangle|$
%where $w=\Dop^2\{f\} \in \Spc M(\R)$ and $p=\sum_{n=1}^{2} \langle f, \phi_n \rangle p_n  \in \Spc N_{\Dop^2}$ uniquely determine the components of the direct-sum decomposition of $f$ as
%\begin{align}
%\label{Eq:directsum}
%f=\Op G_\V \phi w +p,
%\end{align}
and $\Op G_\V \phi: w \mapsto \int_\R g_{\V \phi}(\cdot,y) w(y)\dint y$, with
\begin{align}
\label{Eq:kernel}
g_{\V \phi}(x,y)&=(x-y)_+ - %\sum_{n=1}^2 p_n(x)\langle \phi_n, (\cdot-y)_+\rangle
p_1(x)\langle \phi_1, (\cdot-y)_+\rangle - p_2(x)\langle \phi_2, (\cdot-y)_+\rangle.
\end{align}
\end{proposition}

Central to our formulation  is the unique operator $\Op G_\V \phi:  \Spc M(\R) \to {\rm BV}^{(2)}(\R)$ %whose kernel is given by \eq{Eq:kernel} and
such that
\begin{align}
\label{Eq:rightinv}\Dop^2 \Op G_\V \phi\{ w\}= w & \qquad(\mbox{right-inverse property})\\
\label{Eq:boundary}\langle \phi_1, \Op G_{\V \phi}\{ w\} \rangle %=\Lop^{-1}_\V \phi(0)
=0, \quad \langle\phi_2, \Op G_\V \phi \{w\}\rangle =0 & \qquad(\mbox{boundary conditions})
\end{align}
for all $w \in \Spc M(\R)$. Specifically, \eq{Eq:boundary} ensures the orthogonality of the two components of the direct sum decomposition of $f$ in \eq{Eq:directsum}, while \eq{Eq:rightinv} and the biorthogonality of $(\V \phi,\V p)$ guarantees its unicity.
%
%
%Let $f=g +p$ with $g=\Op G_\V \phi w$ and $p=f-g$. Then, the direct-sum property translates into
%$\Dop^2 f=\Dop^2 g=w$ and $\langle g, \phi_i\rangle=0, i=1,2$ for any $w\in \Spc M(\R)$. 

By fixing $\phi_1=\delta$ and $\phi_2=-\delta+\delta(\cdot-1)$ (finite difference), we obtain the formula of the norm for ${\rm BV}^{(2)}(\R)$ given by \eq{Eq:BV2norm}.
%$$
%\|f\|_{{\rm BV}^{(2)}}\eqdef \|\Dop^2 f \|_{\Spc M} + |f(0)| + |f'(0)|,
%$$
%with the particular choice of $\V \phi=(\delta,-\delta')$. 
The corresponding expression of the kernel of $\Op G_\V \phi$ given by \eq{Eq:kernel} is
\begin{align}
g_{\V \phi}(x,y)&=(x-y)_+ - (1-x)(-y)_+  -x (1-y)_+.
%\left\{
%\begin{array}{rc}
%(x-y) \One_{(0,x]}(y),  &  x\ge 0 \\
% -(x-y) \One_{(x,0]}(y),   &   x<0
%\end{array}
%\right.
\label{Eq:kernelInv}
\end{align}
A crucial observation for the proof of Lemma  \ref{Theo:L1spline} is that the function $y \mapsto g_{\V \phi}(x,y)$ specified by \eq{Eq:kernelInv} is compactly supported and bounded---in contrast with the
leading term of the expansion $(x-y)_+$ in \eq{Eq:kernel}, which represents the impulse response of the conventional shift-invariant inverse of $\Dop^2$ (two-fold integrator). In fact, these functions are continuous, triangle-shaped B-splines with the  following characteristics (see Figure \ref{Fig:kernel}).
\begin{itemize}
\item for $x \le 0$: $y \mapsto g_{\V \phi}(x,y)$ is supported in $[x,1]$ and takes its maximum at $y=0$
\item for $x\in(0,1)$: $y \mapsto g_{\V \phi}(x,y)$ is supported in $[0,1]$ and takes its extremum at $y=x$
\item for $x\ge1$: $y \mapsto g_{\V \phi}(x,y)$ is supported in $[0,x]$ and takes its maximum at $y=1$.
\end{itemize}

\begin{figure}
\centering
\includegraphics[width=8.5cm]{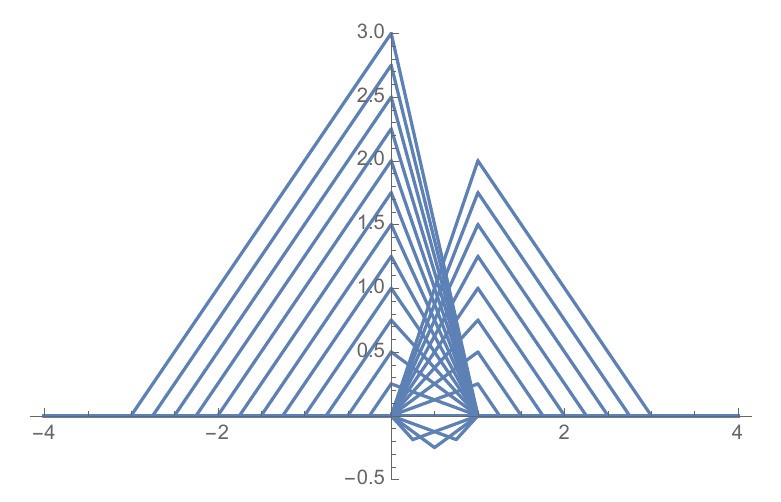}
\caption{\label{Fig:kernel} \small {Graphs of the function $y\mapsto g_{\V \phi}(x,y)$
for a series of values of $x=-3, \dots,3$ with steps of $0.25$.  
This illustrates the property that $g_{\V \phi}(x,\cdot)\in C_0(\R)$ %(i.e., continuity, boundedness and decay at $\pm\infty$) 
for any $x\inR$, which 
is  critical to the proof of Lemma  \ref{Theo:L1spline}. By contrast, the canonical solution of \citep{Unser2017} with $\V \phi=(\delta,-\delta')$ would have resulted in a series of piecewice-linear functions with a discontinuous drop to $0$ at $x=0$.}}
\label{Fig:gphi}
\end{figure}

Since ${\rm BV}^{(2)}(\R)$ is non-reflexive, the characterization of its predual is required for testing the hypothesis of weak* continuity.
%, which is fundamental to our formulation because the space ${\rm BV}^{(2)}(\R)$ is non-reflexive.
To that end, we first recall that the predual of $\Spc M(\R)=\big(C_0(\R)\big)'$ is the space $C_0(\R)$ of continuous functions that vanish at infinity equipped with the sup norm \citep{Rudin1987}. Moreover, since $\Spc S(\R)$ (Schwartz' space of smooth and rapidly-decaying functions)  is dense in $C_0(\R)$ \citep{Schwartz:1966}, the latter can also be described as the completion of $\Spc S(\R)$ equipped with the sup norm, in conformity with the definition of $\Spc M(\R)$ given by \eqref{Eq:Mbyduality}. 

We now present an explicit construction and characterization of the predual of ${\rm BV}^{(2)}(\R)$.
This description is consistent with an earlier theorem of ours \cite[Theorem 6]{Unser2017} applicable to general spline spaces; however, it contributes two novel elements: (i) the operational criterion for space membership provided by the first property,  and (ii) the construction of the predual space $C_{\Dop^2,\V \phi}(\R)$ via the completion of 
$\Spc S(\R)$, which requires additional hypotheses on $\V \phi$.
%Based on Theorem 6 in \citep{Unser2017}, we get a similar characterization for the predual of ${\rm BV}^{(2)}(\R)$.
%\begin{proposition}[Predual space] The native space ${\rm BV}^{(2)}(\R)=\big( C_{\Dop^2,\V \phi}(\R)\big)'$ in Proposition \ref{Prop:BanachNative} is the continuous dual of
%$C_{\Dop^2,\V \phi}(\R)$, which is itself the completion of $\Spc S(\R)$ equipped with the norm
%\begin{align}
%\label{Eq:Dualnorm}
%\|f^\ast\|'_{\rm BV^{(2)}}\eqdef \|\Lop_{\V \phi}^{-1\ast}f^\ast\|_\infty+|\langle p_1, f^\ast\rangle| + |\langle p_2, f^\ast\rangle| 
%\end{align}
%where $\Lop_{\V \phi}^{-1\ast}=\big(\Lop_{\V \phi}^{-1}\big)^\ast$ is the adjoint of $\Lop_{\V \phi}^{-1}$ specified by \eq{Eq:kernel}. Moreover, every  $f^\ast \in C_{\Dop^2,\V \phi}(\R)$
%has the unique direct-sum decomposition
%$$
%f^\ast = \Dop^2 \varphi + \phi,
%$$
%where $\varphi=\Op G^{\ast}_\V \phi\{f^\ast\} \in C_0(\R)$ and $\phi=\sum_{n=1}^{2} \langle f^\ast, p_n \rangle \phi_n  \in {\rm span}\{\phi_1,\phi_2\}$.
%\end{proposition}

\begin{definition}
\label{Def:Admissible}
 Let $\V p=(p_1,p_2)$ be a basis of $\Spc N_{\Dop^2} ={\rm span}\{1,x\}$ and
$\V \phi=(\phi_1,\phi_2)$ a complementary set of (generalized) functions whose Fourier transforms are denoted by $\widehat\phi_1, \widehat\phi_2$. Then, the system $(\V p, \V \phi)$ 
is said to be admissible for $\Dop^2$ if
\begin{enumerate}
\item the basis functions are biorthogonal; i.e., $\langle \phi_m,p_n \rangle=\delta_{m,n}, \ (m,n=1,2)$
\item $\widehat \phi_1,\widehat \phi_2 \in L_{1,2}(\R)=\{f: \R \mapsto \R \ \ \big| \ \int_\R 
(1+|\omega|)^{-2}|f(\omega)|\dint \omega<\infty\}$ with the two functions being continuously differentiable twice at $\omega=0$.\\
\end{enumerate}
\end{definition}

\begin{theorem}[Predual of native space] 
\label{Theo:Predual}
Let $(\V \phi,\V p)$ be an admissible system in the sense of Definition \ref{Def:Admissible}.  Then,
the function space
\begin{align}
\label{Eq:weakcont}
C_{\Dop^2,\V \phi}(\R)=\{g=\Dop^2 v + a_1\phi_1+a_2 \phi_2: v \in C_{0}(\R), \V a=(a_1,a_2)\inR^2\}
\end{align}
has the following properties:
\begin{enumerate}
\item every $g \in C_{\Dop^2,\V \phi}(\R)$ has a unique direct-sum representation as in \eqref{Eq:weakcont} with $v=\Op G^{\ast}_\V \phi\{g\}$, $a_1=\langle p_1, g\rangle$, and $a_2=\langle p_2, g\rangle$, where $\Op G^\ast_\V \phi$ %=\big(\Lop_{\V \phi}^{-1}\big)^\ast$ 
is the adjoint of $\Op G_\V \phi$ specified by \eq{Eq:kernel};
\item $C_{\Dop^2,\V \phi}(\R)$ is a (non-reflexive) Banach space equipped with the norm
\begin{align}
\label{Eq:Dualnorm}
\|g\|_{C_{\Dop^2,\V \phi}}\eqdef \revise{\|\Op G^\ast_\V \phi g\|_\infty+\|\V p(g)\|_2=\|v\|_\infty+\|\V a\|_2};
\end{align}

\item $C_{\Dop^2,\V \phi}(\R)$ is the predual of ${\rm BV}^{(2)}(\R)$ in Proposition \ref{Prop:BanachNative}; i.e.,
${\rm BV}^{(2)}(\R)=\big( C_{\Dop^2,\V \phi}(\R)\big)'$;
\item  $C_{\Dop^2,\V \phi}(\R)$ is the completion of $\Spc S(\R)$ equipped with the $\|\cdot\|_{C_{\Dop^2,\V \phi}}$-norm.
\end{enumerate}
\end{theorem}
\begin{proof}: The main idea is that the  construction expressed by \eqref{Eq:weakcont} is %that $C_{\Dop^2,\V \phi}(\R)$ %=\Spc U \oplus \Spc N_{\V \phi}$ is 
the direct sum of two linear spaces, $\Spc U$ and $\Spc N_{\V \phi}$,
whose Banach topology and completion properties are revealed next.\\[1ex]
 {\em (i): Topology of the space $\Spc N_{\V \phi}={\rm span}\{\phi_1,\phi_2\}$ and of its dual}\\
This %finite-dimensional 
space collects the two last components of $g$ in \eqref{Eq:weakcont} \revise{and is equipped with the discrete $\ell_2$-norm
$\|\phi\|_{\Spc N_{\V \phi}}=\|\M a\|_2$
with $\M a= \V p(\phi)=(\langle p_1,\phi\rangle, \langle p_2,\phi\rangle)$. We also }specify the projection operator $
%{\rm Proj}_{\Spc N_{\V \phi}}: 
C_{\Dop^2,\V \phi}(\R)\to\Spc N_{\V \phi}$:
$$
{\rm Proj}_{\Spc N_{\V \phi}}\{g\}=%\sum_{n=1}^2 \langle p_n,g\rangle\phi_n .
\langle p_1,g\rangle\phi_1+\langle p_2,g\rangle\phi_2.
$$
\revise{The complementary space is 
%\Spc N_{\Dop^2}={\rm span}\{1,x\}=
$\Spc N_{\V p}\eqdef {\rm span}\{p_1,p_2\}$
equipped with the norm
$
\|p\|_{\Spc N_{\V p}}=\|\V p(p)\|_2=\|\M b\|_2
$
with $\M b=\V \phi(p)=(\langle \phi_1,p\rangle,\langle \phi_2,p\rangle)$.
Thanks to the biorthogonality of $\V \phi$ and $\V p$, for all
$p={b_1 p_1 + b_2 p_2} \in \Spc N_{\V p}$,
we have that 
$$
\|p\|_{\Spc N'_{\V \phi}}=\sup_{\phi \in \Spc N_{\V \phi}: \|\phi\|_{\Spc N_{\V \phi}}\le1}\langle \phi ,p\rangle=\sup_{\M a\inR^2: \|\M a\|_2 \le1}\M a^T \M b=\|\M b\|_2=\|p\|_{\Spc N_{\V p}},
$$
which shows that $\Spc N_{\V p}=\Spc N'_{\V \phi}$ is the continuous dual
%\footnote{The abstract dual of $\Spc N_{\V \phi}$ is a whole equivalence class of functionals that is represented; it is represent here by the single concrete instance that ensures } 
of $\Spc N'_{\V \phi}$.}\\[1ex]
{\em (ii): Range of the operator $\Op G^{\ast} _\V \phi$}\\
To derive the required properties, we restrict the domain of $\Op G^{\ast} _\V \phi$  to the 
subspace
\begin{align*}
\Spc S_{\V p^\perp}(\R)&\eqdef\big\{\psi\in \Spc S(\R) : \langle p_1,\psi\rangle=0, \langle p_2,\psi\rangle=0\} \subset \Spc S(\R).
%&=\big\{\psi\in \Spc S(\R) :  {\rm Proj}_{\Spc N_\V \phi}\{\psi\}=0\}.
\end{align*}
%equipped with the topology of $\Spc S(\R)$.
By using the explicit form \eqref{Eq:kernel} of the kernel of $\Op G _\V \phi$, we find that, for any $\psi \in 
\Spc S_{\V p^\perp}(\R)$,
\begin{align}
%\Op G^{\ast} _\V \phi \{\psi\}(x)&= \int_\R \left(\tfrac{1}{2}|y-x| -q_1(x)p_1(y) -q_2(x)p_2(y) \right)\psi(y)\dint y\nonumber\\
%&=\int_\R \tfrac{1}{2}|y-x|\varphi(y)\dint y-q_1(x)\underbrace{\langle p_1,\psi\rangle}_{=0}  -q_2(x)
\Op G^{\ast} _\V \phi \{\psi\}(x)&= \int_\R \big((y-x)_+ -q_1(x)p_1(y) -q_2(x)p_2(y) \big)\psi(y)\dint y\nonumber\\
&=\int_\R (y-x)_+\psi(y)\dint y-q_1(x)\underbrace{\langle p_1,\psi\rangle}_{=0}  -q_2(x)
\underbrace{\langle p_2,\psi\rangle}_{=0} \nonumber\\
&=\Dop^{-2\ast}\{\psi\}(x) \label{Eq:SimpInv}
\end{align}
where $q_n(y)=\langle \phi_n, (\cdot-y)_+\rangle$
%$q_n(y)=\langle \phi_n, \tfrac{1}{2}|\cdot-y|\rangle$ 
for $n=1,2$, and $\Dop^{-2\ast}$ is the 2-fold %symmetric 
(adjoint) integration operator whose frequency response is 
$-1/\omega^2-\jj \pi \delta'(\omega)=\Fourier\{(-x)_+\}(\omega)$. 
%$-1/\omega^2=\Fourier\{\tfrac{1}{2}|x|\}(\omega)$. 
Based on \eqref{Eq:SimpInv}, we then show that
\begin{align}
\label{Eq:ContinuousG}
\forall \psi \in \Spc S_{\V p^\perp}(\R): \quad \Op G^{\ast} _\V \phi \{\psi\} \in C_0(\R),
\end{align}
which, as we shall see, % in Step ({\em iii}), 
implies the \revise{boundedness} of $\Op G^{\ast} _\V \phi: \Spc S_{\V p^\perp}(\R)\to C_0(\R)\embedC \Spc S'(\R)$. Property \eqref{Eq:ContinuousG} is established by examining the Fourier transform\footnote{We use the product rule $\psi(\cdot)\delta'=\psi(0)\delta'-\psi'(0)\delta$, which follows from the definition of the distribution $\delta': \varphi\mapsto \langle \delta',\varphi\rangle=-\varphi'(0)$.} of $f=\Op D^{-2\ast} \{\psi\}$:
\begin{align}
\label{Eq:f}
\hat f(\omega)=-\widehat \psi(\omega)/\omega^2 -\jj \pi \big(\widehat \psi(0)\delta'(\omega) -\widehat \psi^{(1)}(0)\delta(\omega)  \big)
\end{align}
 with $\widehat \psi=\Fourier\{\psi\}\in \Spc S(\R)$.
Since $\V p(\psi)=\V 0 \Leftrightarrow\widehat \psi(0)=\widehat \psi^{(1)}(0)=0$, we first simplify
\eqref{Eq:f} to $\hat f(\omega)=-\widehat \psi(\omega)/\omega^2$ and then invoke a Taylor series argument to deduce the
continuity of  $\hat f(\omega)$ at $\omega=0$.
This, together
with the boundedness and rapid decay of $\widehat \psi(\omega)$, implies that  $\hat f \in L_1(\R)$.
The announced result---i.e., the continuity, boundedness and decay of $f(x)$ at infinity---then follows from the Riemann-Lebesgue lemma.\\[1ex]
{\em (iii): The Banach topology of $\Spc U$}\\
The definition of $\Spc U$, which corresponds to the first component in \eqref{Eq:weakcont}, is 
$$\ \Spc U\eqdef \{f=\Dop^2 v: v \in C_0(\R)\},$$
equipped with the norm $\|\Dop^2 v\|_\Spc U=\|v\|_\infty$, which establishes an isometric isomorphism with $C_0(\R)$. Our intend now is to prove that $\|f\|_\Spc U=\|\Op G^{\ast}_\V \phi f \|_\infty$ for all $f \in \Spc U$, which is equivalent to showing that $\Op G^{\ast}_\V \phi$ is the inverse of
$\Dop^2: C_0(\R) \to \Spc U$.

We shall achieve this through an extension process that builds upon the 
properties of 
the operator $\Op G^{\ast}_\V \phi$ established in Step ({\em ii}).
We start by considering the semi-norm $\psi\mapsto \|\psi\|_{\tilde{\Spc U}}\eqdef \|\Op G^{\ast}_\V \phi \psi \|_\infty$, which is well defined
over $\Spc S_{\V p^\perp}(\R)\subset\Spc S(\R)$.
Since $\Op G^{\ast}_\V \phi\psi=\Dop^{-2\ast}\psi$ for all $\psi \in \Spc S_{\V p^\perp}(\R)$ and $\Dop^{2}\Dop^{-2\ast} \varphi= \varphi$ any $\varphi \in \Spc S(\R)$,
we have that $\|\Op G^{\ast}_\V \phi \psi \|_\infty=0 \Leftrightarrow \psi=0$, which shows
that $\|\cdot\|_{\tilde{\Spc U}}$ is a norm over $\Spc S_{\V p^\perp}(\R)$, as expected.
This allows us to rephrase the inclusion property from Step ({\em ii}) as: 
$\Op G^{\ast}_\V \phi$ isometrically maps 
$(\Spc S_{\V p^\perp}(\R),\|\cdot\|_{\tilde{\Spc U}})$ to the Banach space $(C_0(\R),\|\cdot\|_{\infty})$,
which is the form suitable for the bounded linear transformation (B.L.T.) extension theorem.

\begin{theorem}[{\rm \citet*[Theorem I.7, p. 9]{Reed1980}}]
Let $\Op G$ be a bounded linear transformation from a normed space $(\Spc X,\|\cdot\|_\Spc X)$ to a complete normed space $(\Spc Y,\|\cdot\|_\Spc Y)$. Then, $\Op G$ has a unique extension to a bounded linear transformation (with the same bound) %, $\tilde{\Op G}$, 
from the completion
of $\Spc X$ to $(\Spc Y,\|\cdot\|_\Spc Y)$.
\end{theorem} 

Consequently, the restricted operator from Step (ii) uniquely extends to an
isometry $\Op G^{\ast}_\V \phi: \tilde{\Spc U} \to C_0(\R)$
where the Banach space $\tilde{\Spc U}$ is the completion of $\Spc S_{\V p^\perp}(\R)$
in the $\|\cdot\|_{\tilde{\Spc U}}$-norm.
The final element is that $\Dop^2 \Op G^{\ast}_\V \phi \psi =\Dop^2 \Dop^{-2\ast} \psi=\psi$ for all $\Spc S_{\V p^\perp}(\R)\subseteq \tilde {\Spc U}$, which indicates that
$\Dop^2$ is the inverse of $\Op G^{\ast}_\V \phi$ on $\Spc S_{\V p^\perp}(\R)$. Since the latter is a dense subset of $\tilde {\Spc U}$, we can extend the property to the entire space,
% and $\Op G^\ast_\V \phi:\tilde{\Spc U} \to C_0(\R)$ is an isometry, we can invoke the bounded inverse theorem \cite[Section 8.3]{Farenick2016} to establish that its unique inverse is precisely $\Dop^2: C_0(\R) \to \tilde{\Spc U}$, 
which ultimately proves that $\Spc U=\tilde{\Spc U}$.\\[1ex]
{\em (iv): $C_{\Dop^2,\V \phi}(\R)=\Spc U \oplus \Spc N_{\V \phi}$}\\
The inclusion $g\in C_{\Dop^2,\V \phi}(\R)$ is equivalent to   $g=f + \phi$ where $f=\Dop^2 v$ with $v\in C_0(\R)$
and $\phi=a_1\phi_1+a_2 \phi_2$. 
The components $(f,\phi)$ are retrieved as $f={\rm Proj}_{\Spc U}\{g\}=\Dop^2\Op G^{\ast}_\V \phi g$
and $\phi={\rm Proj}_{\Spc N_{\V \phi}}\{g\}$. The conditions $\Op G^\ast_\V \phi \phi=0$ and ${\rm Proj}_{\Spc N_{\V \phi}}\{\Dop^2 v\}=0$ for all $\phi\in \Spc N_\V \phi$  and $v \in C_0(\R)$ ensure that
$\Spc U \cap \Spc N_{\V \phi}=\{0\}$ so that the sum is direct. The other relevant identity from Step ({\em iii}) is $f=\Dop^2\Op G^{\ast}_\V \phi f$ for all $f\in \Spc U$. 
Consequently, $C_{\Dop^2,\V \phi}(\R)=\Spc U \oplus \Spc N_{\V \phi}$ is a Banach space equipped with the 
sum norm given by \eqref{Eq:Dualnorm}.\\[1ex]
{\em (v): ${\rm BV}^{(2)}(\R)=\big( C_{\Dop^2,\V \phi}(\R)\big)'$}\\
First, we identify the norm of $\Spc U'$ by
%$\|u^\ast\|_{\Spc U'}=\|\Dop^2u^\ast\|_\Spc M$ for any $u^\ast \in \Spc U'$ 
applying a standard duality argument: %: any $u^\ast \in \Spc U'$,
\begin{align*}
\|u^\ast\|_{\Spc U'}&=\sup_{u \in \Spc U: \|u\|_{\Spc U}\le1}\langle u^\ast , u\rangle=\sup_{v \in C_0(\R): \|v\|_{\infty}\le1}\langle u^\ast , \Dop^2 v\rangle\\
&=\sup_{v \in \Spc S(\R): \|v\|_{\infty}\le1}\langle \Op D^2 u^\ast , v\rangle=\|\Dop^2 u^\ast\|_{\Spc M}
\end{align*}
where we have used the identity $u=\Dop^2 v$ with $v\in C_0(\R)$ and the denseness of $\Spc S(\R)$ in $C_0(\R)$.
The dual of $C_{\Dop^2,\V \phi}(\R)$ in Step ({\em iv}) is then given by $\Spc U' \oplus \Spc N'_{\V \phi}=\Spc U' \oplus \Spc N_{\V p}$ equipped with the sum of the dual norms: $\|(u^\ast,p)\|= \|u^\ast\|_{\Spc U'}+\|p\|_{\Spc N_{\V p,1}}=\| \Dop^2 f\|_{\Spc M} + \|\V \phi(f)\|_1=\|f\|_{{\rm BV}^{(2)}}$.\\[1ex]
{\em (vi): $C_{\Dop^2,\V \phi}(\R)$ is the completion of $\Spc S(\R)$ in the  $\|\cdot\|_{C_{\Dop^2,\V \phi}}$-norm}.\\%[0.4ex]
The idea is to amend the extension technique of Step ({\em iii}) by selecting a second biorthogonal system $(\V \varphi,\V p)$ such that
$\Spc N_{\V \varphi}={\rm span}\{\varphi_1,\varphi_2\} \subset  \Spc S(\R)$. This yields
the direct-sum decomposition
of $\varphi=\tilde \psi + \tilde \phi\in \Spc S(\R)$ with  $\tilde \phi= {\rm Proj}_{\Spc N_{\V \varphi}}\{\varphi\} \in \Spc N_{\V \varphi}$
%={\rm span}\{\varphi_1,\varphi_2\}
%\subset\Spc S(\R)$, 
and $\tilde \psi=\varphi-\tilde \phi \in \Spc S_{\V p^\perp}(\R)$. While
we already know that $\Op G^\ast_\V \phi\tilde \psi\ \in C_0(\R)$,
the delicate point is to make sure that the same holds true
for $\Op G^\ast_\V \phi\tilde \phi$. %, which is not necessarily vanishing anymore. 
Since $\tilde \phi \in {\rm span}\{\varphi_1,\varphi_2\}$, the latter requirement is equivalent to
\begin{align}
\label{Eq:InC0}
\Op G^\ast_\V \phi \{\varphi_n\}=\Op D^{-2\ast}(\Op {Id}-{\rm Proj}_{\Spc N_\V \phi}) \{ \varphi_n\}=\Op D^{-2\ast}\{ \varphi_n-\phi_n\} \in C_0(\R)
\end{align}
for $n=1,2$. With the same arguments as in Step ({\em ii}) (Riemann-Lebesgue lemma), we ensure that \eqref{Eq:InC0} is met
by imposing the Fourier-domain condition
\begin{align}
\label{Eq:InC0b}
\frac{\widehat{\phi}_n(\omega)-\widehat{\varphi}_n(\omega)}{\omega^2} \in L_1(\R),
\end{align}
%While the latter is a slight restriction, it is more practical as it 
which results from the second hypothesis in Definition \ref{Def:Admissible}.
% which does not involve any $\varphi_n$. 
In effect, the role of $\widehat{\varphi}_n\in \Spc S(\R)$ in \eqref{Eq:InC0b} is to temper the singularity of $1/\omega^2$ at the origin, thanks to the condition $\V p(\varphi_n-\phi_n)=\V 0$, which induces a second-order zero in the numerator---this correction does not impact integrability otherwise because
of the rapid decay of $\widehat{\varphi}_n$. 

Having established that $\Op G^\ast_\V \phi \{\tilde \psi+\tilde \phi\}\in C_0(\R)$,
we can now check that
$$
\|\varphi\|_{C_{\Dop^2,\V \phi}}=\|\Op G^\ast_\V \phi \{\tilde \psi+\tilde \phi\} \|_\infty+ \|\V p(\tilde \phi)\|_2=0 \Leftrightarrow (\tilde \psi,\tilde \phi)=(0,0) \Leftrightarrow \varphi=0,
$$
which proves that $\|\cdot\|_{C_{\Dop^2,\V \phi}}$ is a valid norm over $\Spc S(\R)=\Spc S_{\V p^\perp}(\R)\oplus \Spc N_{\V \varphi}$.
We then deduce the desired completion result from the B.L.T. theorem by observing that
$\Op G^\ast_\V \phi: (\tilde \psi,\tilde \phi)\mapsto \Op G^\ast_\V \phi\{\tilde \psi +\tilde \phi\}$ is bounded from
$(\Spc S(\R), \|\cdot\|_{C_{\Dop^2,\V \phi}})$ to $(C_0(\R),\|\cdot\|_\infty)$. (The 
boundedness of the operator simply follows from the inequality
\begin{align*}
%\label{Eq:Dualnormbound}
\|\Op G^\ast_\V \phi \varphi\|_\infty\le \|\varphi\|_{C_{\Dop^2,\V \phi}}= \|\Op G^\ast_\V \phi \varphi\|_\infty+\|\V p(\varphi)\|_2<\infty
%\le \big(C_1 + \sup_{n=1,2}\|p_n\|_{\infty,1}\big) \|\varphi\|_{L_{1,1}}
\end{align*}
for any $\varphi \in \Spc S(\R)$.)
%is implied
%by the continuity of $\Op G^\ast_\V \phi: \Spc S(\R)\to C_0(\R)$ when
%$\Spc S(\R)$ is endowed with the $\|\cdot\|_{C_{\Dop^2,\V \phi}}$-norm.
\end{proof}

By considering the dual form of Property 4 in Theorem \ref{Theo:Predual} (which is a new result to the best of our knowledge), we obtain an alternative, self-contained definition of our native space as
\begin{align}
\label{Eq:BVbyduality}
{\rm BV^{(2)}}(\R)=\{f \in \Spc S'(\R):{ \sup_{\varphi \in \Spc S(\R): \|\varphi\|_{C_{\Dop^2,\V \phi}}\le1}\langle f ,\varphi\rangle}<\infty\}
\end{align}
which is the direct analog of \eqref{Eq:Mbyduality}. Property 4 
actually tells that $\Spc S(\R) \embedC C_{\Dop^2,\V \phi}(\R)$ with the embedding being dense.
%(in conjunction with the observation that $C_{\Dop^2,\V \phi}(\R)\embedC \Spc S'(\R)$)
This, together with the observation that $C_{\Dop^2,\V \phi}(\R)\embedC \Spc S'(\R^d)$,
 implies that $\Spc S(\R) \embedC {\rm BV^{(2)}}(\R)\embedC \Spc S'(\R)$ (by duality)
 with the outer embedding being dense since $\Spc S(\R)$ is itself dense in $\Spc S'(\R)$.
In effect, this means that any ``generalized'' function---and, a fortiori, any continuous function $f: \R \to \R$---can be approximated to an arbitrary precision by a member of ${\rm BV}^{(2)}(\R)$.

Another interesting observation is that the ``canonical'' choice 
$\V \phi=(\delta,-\delta')$ from \citep{Unser2017} does not fulfil the second condition in Definition \ref{Def:Admissible} (it actually fails by a tiny margin because $-(\jj \omega)$ is only in $L_{1,2+\epsilon}(\R)$ for any $\epsilon>0$). This means that Property 4 does not apply to that particular case, even though the underlying native spaces are hardly distinguishable as sets. 
The only significant difference is in the specification of the corresponding weak* topology, which is essential to the proof of Lemma 1.

\section{Proof of Lemma \ref{Theo:L1spline}}
\label{App:ProofLemma}
\begin{proof} %[Proof of Lemma \ref{Theo:L1spline}]
The lemma is deduced from \cite[Theorem 4]{Unser2017}: an abstract optimality result for generalized spline interpolation   that holds for an extended class of admissible regularization operators $\Lop$ and for arbitrary linear functionals ($\nu_m: f\mapsto\langle\nu_m,f \rangle$), subject to the weak* continuity requirement. 
The relevant version of the result for functions $f: \R \to \R$ is restated here in the explicit form of Theorem \ref{Theo:L1spline2}. 

The maximal polynomial rate of growth ($n_0$) of functions is controlled via their inclusion in the space
$$\revise{L_{\infty,n_0}(\R)=\{f: \R \to \R \ \mbox{ s.t.}\  \|f\|_{\infty,n_0}\eqdef \esssup_{x\inR}(1+|x|)^{-n_0}|f(x)| <\infty\}.}
$$

 \begin{definition} [Spline-admissible operator]
\label{Def:splineadmis}
A linear operator $\Op L: \Spc M_\Lop(\R) \to \Spc M(\R)$, where $\Spc M_\Lop(\R)\supset \Spc S(\R)$ is an appropriate subspace of $\Spc S'(\R)$, is called {\em spline-admis\-sible}
if 
\begin{enumerate}
\item it is shift-invariant;% that is, $\Op L\{s(\cdot-\bx_0)\}=\Op L\{s\}(\cdot-\bx_0)$ for any signal $s \in \Spc M_\Lop(\R^d)$;
\item there exists a function $\rho_{\Op L}: \R \to\R $ of slow growth (the Green's function of $\Op L$) such that
$\Lop\{\rho_{\Op L}\}=\delta$, where $\delta$ is the Dirac impulse. The rate of polynomial growth of 
$\rho_{\Op L}$ is $n_0=\inf\{n \in \N: \rho_\Lop \in L_{\infty,n}(\R)\}$.
\item the (growth-restricted) null space of $\Lop$, $$\Spc N_\Lop=\{q \in L_{\infty,n_0}(\R): \Lop\{q\}=0\},$$
has the finite dimension $N_0\ge0$.
% and maximal order of polynomial growth $n_0\in\{0,\ldots,N_0-1\}$;
\end{enumerate}
\end{definition}
The native space of $\Lop$, $\Spc M_{\Lop}(\R)$, is then identified as
\begin{align}
\label{Eq:native}
\Spc M_{\Lop}(\R)=\{f \in L_{\infty,n_0}(\R): \|\Lop f\|_{\Spc M}<\infty\}.
\end{align}
In addition, it is assumed that $\Spc M_{\Lop}(\R)$ is equipped with an appropriate Banach topology which gives a concrete meaning to the underlying notion of (weak*-) continuity.

\revise{ As expected, the operator $\Lop=\Dop^2$ is spline-admissible:
% in the sense of Definition \ref{Def:splineadmis}:
%\cite[see Definition 1]{Unser2017}: 
Its causal Green's function is $\rho_{\Dop^2}(x)=(x)_+$ (ReLU) which exhibits the algebraic rate of growth $n_0=1$, while its null space $\Spc N_{\Dop^2}={\rm span}\{p_1, p_2\}$ with
$p_1(x)=1$ and $p_2(x)=x$ is finite-dimensional with $N_0=2$.
%and $p_1(x)=1, p_2(x)=x$. 
These are precisely the basis functions associated with $\Lop$ that appear in \eq{Eq:spline}.}

\revise{ We now show that the slow growth condition with $n_0=1$ is implicit
in the specification of ${\rm BV}^{(2)}(\R)$ given by \eqref{Eq:Native} and/or Proposition \ref{Prop:BanachNative} so that our definition of the native space 
is consistent with \eqref{Eq:native}. % and the choice $\Lop=\Dop^2$.
\begin{proposition}
\label{Prop:NativeEmbbed}
With the choice of topology specified in Appendix B,
${\rm BV}^{(2)}(\R) \embedC L_{\infty,1}(\R)$, while
$$f \in {\rm BV}^{(2)}(\R) \quad \Leftrightarrow \quad {\rm TV}^{(2)}(f)\eqdef \sup_{\|\varphi\|_\infty\le 1: \varphi \in \Spc S(\R)} \langle f, \Dop^2 \varphi \rangle=\|\Dop^2 f\|_{\Spc M}<\infty.$$
\end{proposition}
\begin{proof}
The key is the bound $\|g_{\V \phi}(x,\cdot)\|_\infty\le |x|$ for any $x\inR$ (see Figure \ref{Fig:gphi} and accompanying explanations), which implies
that
$$
C_\V \phi=\esssup_{x,y \inR} (1+|x|)^{-1} |g_{\V \phi}(x,y)|<\infty.
$$
This ensure the continuity of the operator $\Op G_\V \phi: \Spc M(\R) \to L_{\infty,1}(\R)$ 
with $\|\Op G_\V \phi\|=C_\V \phi$ by \cite[Theorem 3]{Unser2017}.
Next, we use the property that any $f\in{\rm BV}^{(2)}(\R)$
admits a unique decomposition $f=\Op G_\V \phi w + p$ with $w=\Lop f\in \Spc M(\R)$
and $p=\sum_{n=1}^{2} \langle \phi_n,f\rangle p_n \in \Spc N_\V p$,
so that
\begin{align*}
\|f \|_{\infty,1}& \le \|\Op G_\V \phi w\|_{\infty,1} + \|p\|_{\infty,1} \\
& \le C_\V \phi \|w\|_{\Spc M} + \sum_{n=1}^{2} |\langle \phi_n,f\rangle|\; \|p_n\|_{\infty,1} \\
& \le C_\V \phi \|\Lop f\|_{\Spc M} + \|\V p(f)\|_2 \sum_{n=1}^{2} \|p_n\|_{\infty,1} \\
& \le \left(C_\V \phi + \sum_{n=1}^{2} \|p_n\|_{\infty,1}\right) \|f\|_{{\rm BV}^{(2)}},
\end{align*}
which proves that ${\rm BV}^{(2)}(\R)$ is continuously embedded in $L_{\infty,1}(\R)$.
The reason for using the dual definition of the ${\rm TV}^{(2)}$ semi-norm in the last statement of the proposition is that the formula remains valid for any $f \in \Spc S'(\R)$
with ${\rm TV}^{(2)}(f)=\infty \Leftrightarrow f \notin {\rm BV}^{(2)}(\R)$. Likewise,
${\rm TV}^{(2)}(f)=0 \Leftrightarrow f \in \Spc N_\V p$.
\end{proof}
}

 \begin{theorem}[Generalized spline interpolant]
\label{Theo:L1spline2}
Let us as\-sume that the following conditions are met:
\begin{enumerate}
\item The operator ${\Lop:\Spc M_{\Lop}(\R)\to\Spc M(\R)}$ is spline-admissible in the sense of Definition \ref{Def:splineadmis}. 
\item The linear measurement operator $\V \nu: f \mapsto \V \nu(f)=\big(\langle \nu_1, f\rangle, \ldots, \langle \nu_M, f\rangle\big)$ maps $\Spc M_{\Lop}(\R^d) \to \R^M$ and is weak*-continuous on $\Spc M_\Lop(\R^d)=\big( C_\Lop(\R^d)\big)'$.
%\footnote{This is equivalent to $\nu_m \in C_\Lop(\R^d)$: the predual of $\Spc M_\Lop(\R^d)$ (see Proposition \ref{Theo:Predual})}.
\item {The recovery
problem is well-posed over the null space of $\Lop$: $\V \nu(q_1)=\V \nu(q_2) \Leftrightarrow q_1=q_2$, for any $q_1,q_2 \in \Spc N_\Lop$.}
\end{enumerate}% and satisfies the stability condition in Theorem \ref{Theo:gBeppoLevi}.
%This ensures that $\Spc M_{\Lop}(\R^d)$ is a Banach space equipped with the norm
%Let $\Spc M_{\Lop}(\R^d)$ be the Banach space
%\footnote{The implicit assumption here is that
%$\Lop$ admits
%a well-defined regularized right inverse $\Lop^{-1}:\Spc M(\R^d)\to \Spc M_{\Lop}(\R^d)\embedC L_{\infty,n_0}(\R^d)$ where $0\le n_0\le N_0$ denotes the maximum rate of polynomial growth of the null space components $p_n\in \Spc N_{\Lop}={\rm span}\{p_n \}_{n=1}^{N_0}$.} associated with some spline-admissible operator $\Lop$ 
%equipped with the norm
%$$\|f\|_{\Spc M_\Lop}=\|\Lop f\|_{\Spc M}+\sqrt{\sum_{k=1}^{N_0} |\langle \phi_k, f\rangle|^2}$$ 
%where $\{\phi_n\}$ is a suitable set of analysis functions.
%We consider the continuous linear map $\V \nu$ from $\Spc M_{\Lop}(\R^d) \to \R^M: f \mapsto \V \nu(f)=\big(\langle \nu_1, f\rangle, \ldots, \langle \nu_M, f\rangle\big)$. We also assume that t
%
%\begin{align}
% \forall  f \in \Spc M_{\Lop}:\quad& \|\V \nu(f)\|\le A \|f\|_{\Spc M_{\Lop, \V \phi} },\label{Eq:recCond1}\\
%\forall  p \in \Spc N_{\Lop}:\quad&B \|p\|_{\Spc M_{\Lop, \V \phi}}  \le\|\V \nu(p)\| \label{Eq:recCond2}\
%\end{align}
%for some constants $A,B>0$, where the underlying norm $\|\cdot\|_{\Spc M_{\Lop, \V \phi} }$ is specified by   \eqref{Eq:norm}.
%$$\|f\|_{\Spc M_\Lop}=\|\Lop f\|_{\Spc M}+\sqrt{\sum_{k=1}^{N_0} |\langle \phi_k, f\rangle|^2}.$$ 
% and every $f\in \Spc M_{\Lop}$ and $p \in \Spc N_{\Lop}$. 
Then, the extremal points of the (feasible) generalized interpolation problem
\begin{align}
\label{eq:genproblem2}
\beta= \min_{f \in \Spc M_{\Lop}(\R)} \|\Lop f\|_{\Spc M}\quad \mbox{   s.t.   } \quad \V \nu(f)=\M y
\end{align}
are necessarily nonuniform $\Lop$-splines of the  form
\begin{align}
\label{eq:spline}
s(x)=\sum_{n=1}^{N_0} b_n p_n(x) + \sum_{k=1}^K a_k \rho_\Lop(x-\tau_k)
\end{align}
with parameters $\M b=(b_1,\ldots,b_{N_0})\in \R^{N_0}$, $K\le M-N_0$ (effective number of knots), $\{\tau_k\}_{k=1}^K$ with $\tau_k \in \R$,  and $\M a=(a_1,\ldots,a_K)\in \R^K$. Here,
$\{p_n\}_{n=1}^{N_0}$ is a basis of $\Spc N_\Lop$ and $\Lop\{\rho_\Lop\}=\delta$ so that $\beta=\|\Lop s\|_{\Spc M}=\sum_{k=1}^K |a_k|=\|\M a\|_1$. The full solution set of \eqref{eq:genproblem2} is the weak$\ast$-closed convex hull of those extremal points.
\end{theorem}
%Here too, the proof is deferred to Section \ref{Sec:ProofgTVTheorems}.

% that holds for the so-called class of spline-admissible regularization operators 
%$\Lop$ and for arbitrary linear functionals ($\nu_m$), subject to the weak* continuity requirement. 

Hence, we only need to show that the underlying mathematical
hypotheses are met for the spline-admissible operator $\Lop=\Dop^2$ and $\nu_m=\delta(\cdot-x_m)$:
\begin{itemize}
%\item Spline-admissibility. 
%\item Banach-structure of native space: As detailed in the appendix, our native space can be equipped with a norm. This results in the equivalent definition
%$$
%{\rm BV}^{(2)}(\R) =\{ f: \R \to \R : \|f\|_{{\rm BV}^{(2)}} \eqdef \|\Dop^2 f\|_{\Spc M} + |f(0)| +|f'(0)|< \infty\}
%$$
%where the additional term $|f(0)| +|f'(0)|$ is necessary to control the null-space component of $f$ whose regularization cost is zero.
\item weak* continuity of sampling functionals with respect to the topology specified in Appendix 
\ref{App:BVTopology}
with $\phi_1=\delta$ and $\phi_2=-\delta+\delta(\cdot-1)$.
\begin{proposition}
\label{Prop:samplingContinuous}
The sampling functional $\delta(\cdot-x_m): f \mapsto f(x_m)$
is weak*-continuous on ${\rm BV}^{(2)}(\R)=\big( C_{\Dop^2,\V \phi}(\R)\big)'$ for any $x_m\inR$. Moreover, it satisfies the continuity bound 
$$
|f(x_m)|=|\langle \delta(\cdot-x_m),f\rangle|\le (1 + 2 |x_m|) \ \|f\|_{{\rm BV}^{(2)}},
$$
for any $f \in {\rm BV}^{(2)}(\R)$. 
%The sampling functional $\delta(\cdot-x_m): f \mapsto \langle  \delta(\cdot-x_m),f\rangle=f(x_m)$ is weak*-continuous on ${\rm BV}^{(2)}(\R)=\big(C_{\Dop^2}(\R) \big)'$ for any $x_m\inR$ with continuity bound 
%\begin{align*}
%|f(x_m)| = |\langle  \delta(\cdot-x_m),f\rangle|&\le (1 + 2 |x_m|) %\left( \|\Dop^2 f\|_{\Spc M}+ |f(0| + |f'(0)| \right)\\
%%& (1 + 2 |x_k|)
%  \ \|f\|_{{\rm BV}^{(2)}}
%\end{align*}
%for any $f \in {\rm BV}^{(2)}(\R)$.

\end{proposition}
\begin{proof} The key here is that $\Op G_{\V \phi}^{\ast}\{\delta(\cdot-x_m)\}(x)=g_{\V \phi}(x_m,x)$ where the latter kernel---defined by \eq{Eq:kernelInv}---is continuous, bounded and compactly-supported (see Figure \ref{Fig:kernel} and accompanying explanations), %in Appendix \ref{App:BVTopology}), 
and hence vanishing at $\pm\infty$.
Consequently, $\delta(\cdot-x_m)=\Dop^2v + a_1\phi_1+a_2\phi_2$ with $v=g_{\V \phi}(x_m,\cdot)\in C_0(\R)$, $a_1=\langle 1, \delta(\cdot-x_m)\rangle=1$, and $a_2=
\langle x, \delta(\cdot-x_m)\rangle=x_m$ in accordance with \eqref{Eq:weakcont} in Theorem \ref{Theo:Predual}, which proves that $\delta(\cdot-x_m)\in C_{\Dop^2,\V \phi}(\R)$. This establishes its weak* continuity on $\big(C_{\Dop^2,\V \phi}(\R)\big)'$
(by Proposition \ref{Prop:WeakStar}). %={\rm BV}^{(2)}(\R)$ 

Based on the observation that $\|g_{\V \phi}(x_m,\cdot)\|_\infty\le |x_m|$, we then easily estimate
the norm of $\delta(\cdot-x_m)$ as
\begin{align*}
\|\delta(\cdot-x_m)\|'_{\rm BV^{(2)}}&=\sup_{y \inR} |g_{\V \phi}(x_m,y)| \ + \sup_{n=1,2}|a_n| \\
%&=\sup_{y \inR} |g_{\V \phi}(x_m,y)| \  + 1 \ + \ |x_m| \\ &
& \le | x_m| \  + 1 \ + \ |x_m|  <\infty.
\end{align*}
Finally, we recall that the property that two Banach spaces $\Spc U$ and $\Spc U'$ form a dual pair implies that
$|\langle  u,u'\rangle|\le \|u\|_{\Spc U} \|u'\|_{\Spc U'}$ for any $u \in \Spc U$ and $u' \in \Spc U'$.
Taking $\Spc U=C_{\Dop^2,\V \phi}(\R)$ and $u=\delta(\cdot-x_m)$ allows us to 
translate the above norm estimate into the announced continuity bound.
\end{proof}
 
\item Well-posedness of reconstruction  for $f \in \Spc N_{\Dop^2}={\rm span}\{1,x\}$. It is well-known that the classical linear regression problem
$$\V b=\arg \min_{b_1,b_2} \sum_{m=1}^M |y_m-  (b_1 + b_2 x_m)|^2$$ is well posed and has a unique solution if and only if $S=\{x_m\}_{m=1}^{M}$ contains at least
two distinct points, say $x_{1}\ne x_{2}$, which takes care of the final hypothesis in 
Theorem \ref{Theo:L1spline2}.
%\cite[Theorem 4]{Unser2017}.
\end{itemize}
\end{proof}
%\pagebreak
\subsection*{Acknowlegdments}
The research was partially supported by the Swiss National Science
Foundation under Grant 200020-162343.
% and the European Commission under Grant ERC-2010-AdG 267439-FUN-SP.
The author is thankful to Julien Fageot, Shayan Aziznejad, Anais Badoual, Kyong Hwan Jin, and Harshit Gupta for helpful discussions.

%\bibliographystyle{plain}
%\bibliographystyle{pnas}
%\bibliography{/Users/unser/Bibliography/Unser}
\bibliography{/GoogleDrive/Bibliography/BibTex_files/Unser}

%\newpage

\end{document}